\documentclass[10pt,journal,compsoc]{IEEEtran}

\usepackage{cite}
\usepackage{xcolor}

\usepackage{amsmath}
\usepackage{caption}
\usepackage[justification=centering]{subcaption}
\usepackage{algorithm}
\usepackage{algorithmic}

\usepackage{enumitem}
\usepackage[normalem]{ulem}

\makeatletter
\newcommand{\removelatexerror}{\let\@latex@error\@gobble}
\makeatother

\usepackage{url}

\usepackage[pagebackref,breaklinks,colorlinks]{hyperref}
\usepackage{booktabs}
\usepackage{multirow} 
\usepackage{graphicx}
\usepackage{amsfonts}
\usepackage{amssymb}
\usepackage{bm}
\usepackage{amsthm}

\newtheorem{proposition}{Proposition}

\newtheorem{remark}{Remark}

\theoremstyle{definition}
\newtheorem{definition}{Definition}
\newtheorem{example}{Example}

\usepackage{egothic}

\newcommand{\Cube}{\small{\textegoth{C}}}

\begin{document}
%\title{ACM: Accelerated Consensus Maximization for Geometric Vision By A Unified Approach}
\title{Accelerating Globally Optimal Consensus Maximization in Geometric Vision}

% Authors 
\author{Xinyue~Zhang,~\IEEEmembership{Student Member,~IEEE,}
        Liangzu~Peng,
        Wanting~Xu,~\IEEEmembership{Student Member,~IEEE,}
        and~Laurent~Kneip,~\IEEEmembership{Senior~Member,~IEEE}% <-this % stops a space
\IEEEcompsocitemizethanks{\IEEEcompsocthanksitem X.~Zhang, W.~Xu and L.~Kneip are with the Mobile Perception Lab, ShanghaiTech University. L.~Kneip is also with the Shanghai Engineering Research Center of Intelligent Vision and Imaging. \protect\\
E-mails: see https://mpl.sist.shanghaitech.edu.cn/contact.html
}% <-this % stops an unwanted space
% \thanks{Manuscript received April XX, 2023; revised August XX, 2023.}
\thanks{© 2024 IEEE.  Personal use of this material is permitted.  Permission from IEEE must be obtained for all other uses, in any current or future media, including reprinting/republishing this material for advertising or promotional purposes, creating new collective works, 
for resale or redistribution to servers or lists, or reuse of any copyrighted component of this work in other works.}
}

\markboth{ }%
{Zhang \MakeLowercase{\textit{et al.}}: ACM: Accelerated Consensus Maximization}

% The abstract
\IEEEtitleabstractindextext{%
\begin{abstract}
  Branch-and-bound-based consensus maximization stands out due to its important ability of retrieving the globally optimal solution to outlier-affected geometric problems. However, while the discovery of such solutions caries high scientific value, its application in practical scenarios is often prohibited by its computational complexity growing exponentially as a function of the dimensionality of the problem at hand. In this work, we convey a novel, general technique that allows us to branch over an $n-1$ dimensional space for an n-dimensional problem. The remaining degree of freedom can be solved globally optimally within each bound calculation by applying the efficient interval stabbing technique. While each individual bound derivation is harder to compute owing to the additional need for solving a sorting problem, the reduced number of intervals and tighter bounds in practice lead to a significant reduction in the overall number of required iterations. Besides an abstract introduction of the approach, we present applications to four fundamental geometric computer vision problems: camera resectioning, relative camera pose estimation,  point set registration, and rotation and focal length estimation. Through our exhaustive tests, we demonstrate significant speed-up factors at times exceeding two orders of magnitude, thereby increasing the viability of globally optimal consensus maximizers in online application scenarios.
\end{abstract}

}

\maketitle

\IEEEdisplaynontitleabstractindextext
\IEEEpeerreviewmaketitle

%%%%%%%%%%%%%%%%%%%%%%%%%%%%%%%%%%%%%%%%%%%%%%%%%%%%%%%%%%%%%%%%%%%%%%%%%%%%%%%%%%%%%%
%%%%%%%%%%%%%%%%%%%%%%%%%%%%%%%%%%%%%%%%%%%%%%%%%%%%%%%%%%%%%%%%%%%%%%%%%%%%%%%%%%%%%%
%%%%%%%%%%%%%%%%%%%%%%%%%%%%%%%%%%%%%%%%%%%%%%%%%%%%%%%%%%%%%%%%%%%%%%%%%%%%%%%%%%%%%%
%%%%%%%%%%%%%%%%%%%%%%%%%%%%%%%%%%%%%%%%%%%%%%%%%%%%%%%%%%%%%%%%%%%%%%%%%%%%%%%%%%%%%%
\IEEEraisesectionheading{\section{Introduction}\label{Sec:introduction}}

\IEEEPARstart{C}{onsensus} Maximization (CM) plays an important role in the field of Geometric Computer Vision. Given a set of pair-wise correspondences, the task is to find a solution that is geometrically consistent. Important examples include camera resectioning and epipolar geometry, where we are given input correspondences between 3D world points and 2D image points, or simply 2D image points drawn from two distinct images. However, rather than just having to find an optimal solution to absolute or relative camera pose by minimizing an integral error taken over all correspondences, problems in geometric computer vision are often complicated by the presence of outliers in the data. The majority of solvers employs least sum-of-squares objectives, which are easily disturbed by outliers. In the outlier affected case, the objective is therefore changed to the maximization of the number of correspondences for which the error falls below a pre-defined error threshold, the latter often being chosen as a function of the natural noise in the data. We refer to this as the Consensus Maximization (CM) problem.

Mathematically, one may formulate the CM problem as follows. Let $f:\mathcal{C}\times \mathcal{D} \to [0,\infty)$ be a residual function derived from the problem to be solved. $\mathcal{C} \subset \mathbb{R}^n$ denotes the constraint set that contains the desired solution variable $\mathbf{b}^*$,  and $\mathcal{D}$ is the data domain containing data samples $\mathbf{s}_i$. The entire, corrupted set of sample data is given by $\mathcal{S} = \{ \mathbf{s}_i\}_{i = 1}^M \subset \mathcal{D}$. With some threshold $\epsilon$, a sample $\mathbf{s}_i$ is called an inlier if $f(\mathbf{b}, \mathbf{s}_i)< \epsilon$, or an outlier otherwise. The general form of CM is then given by the optimization problem
\begin{equation}
    \begin{aligned}
        \max_{\mathbf{b}, \mathcal{I} } \; & \;|\mathcal{I}| \\
        \text{s.t.} \; & \; f(\mathbf{b}, \mathbf{s}_i) \leq \epsilon, \\
        & \; \forall \;\mathbf{s}_i \in \mathcal{I} \subseteq \mathcal{D} 
    \end{aligned}
    \label{Problem: CM}
\end{equation}
Here $| \cdot |$ denotes the number of elements in a set (cardinality) and $\mathcal{I}$ can be viewed as a set of inliers with respect to $\mathbf{b}$.

Given the count-based nature of CM, the objective is generally discrete and can no longer be solved using traditional optimization approaches. Instead, there are two kinds of CM approaches that have been proposed over the years: Randomized and deterministic. Randomized consensus maximization employs iterative random sampling of small subsets to generate model hypotheses $\mathbf{b}$ from least-squares solvers. In each iteration, we then find all correspondences for which $f(\mathbf{b}, \mathbf{s}_i) \leq \epsilon$, and thus determine the cardinality of the inlier subset $\mathcal{I}$ that agrees with the calculated hypothesis. If an outlier-free subset is sampled, we tendentially find a good hypothesis and an inlier ratio that is close to the true inlier ratio. Termination criteria typically depend on the expected and current best inlier ratio. The method returns the solution $\mathbf{b}$ corresponding to the largest cardinality inlier subset found during the iterations. Simple as it sounds, randomized CM methods such as RANSAC~\cite{fischler1981random} have been a milestone for robust geometric model estimation, and they are in broad use today. However, Randomized CM suffers from the inability to guarantee global optimality within a finite number of iterations. The expected number of iterations for RANSAC to reach a predefined confidence level grows significantly as the outlier ratio increases.

This is where Deterministic CM comes into play~\cite{li2009consensus,chin_efficient_2015,cai2018deterministic}. Instead of seeking consensus via randomization, deterministic methods exhaustively search the whole solution space, thereby guaranteeing global optimality. Several paradigms for the global search exist, such as Branch and Bound (BnB)~\cite{li2009consensus}, $A^*$ tree search~\cite{chin_efficient_2015}, and bisectioning~\cite{cai2018deterministic}, to just name a few. The global search provides an accurate solution and is able to handle extreme outlier ratios. Though computationally demanding, deterministic CM therefore remains of crucial scientific value especially as a reference implementation to test alternative (e.g. randomized) CM methods in scenarios in which no ground truth solution is available (e.g. on real, outdoor data).

We address the computational efficiency of deterministic CM by introducing Accelerated deterministic Consensus Maximization (ACM), a novel method that can be widely applied to speed up BnB-based globally optimal CM without depending much on the specific nature of a solved problem. This contribution is important as globally optimal, deterministic methods quickly become computationally intractable as the dimensionality of the problem increases. The detailed contributions and paper structure are as follows.

\subsection{Contributions}

\begin{itemize}
    \item We propose ACM, a general and flexible technique for accelerating deterministic consensus maximization. The method proceeds by performing a 1-DoF reduction of the space over which BnB is branching. The remaining dimension is solved efficiently and globally optimally using interval stabbing.
    \item The bounds achieved via this technique are not only faster to compute, but also tighter.
    \item We apply ACM on 5 different problems covering scenarios from $1$ dimension to $4$ dimension and obtain a $2 \times$-$100 \times$ speed-up (and sometimes exceeding $300\times$ speed-up).
\end{itemize}

\subsection{Paper Structure}

The rest of this work is organised as follows: 
Section~\ref{Sec:RelatedWork} reviews important related literature on both randomized and deterministic consensus maximization methods as well as existing approaches to increase their computational efficiency.
Section~\ref{Sec:BnB} reviews the Branch and Bound method in computer vision and the key technique of interval stabbing.
Section~\ref{Sec:ACM-X} contains our core contribution: Accelerated Consensus Maximization. It introduces the main technique for 1-DoF dimensionality reduction, requirements for the approach to be applicable, and important properties such as computational complexity.
Next, Sections~\ref{Sec:1DProblem},~\ref{Sec:2DProblem},~\ref{Sec:3DProblem} and~\ref{sec:4D} present applications on the $1$-dimensional problem of IMU-supported camera localization, the $2$-dimensional problem of relative pose under a planar motion assumption, 
 the $3$-dimensional problems of Correspondence-based and Correspondence-less point set registration, 
and $4$-dimensional problem of rotation and focal Length estimation respectively. We conclude with a discussion in Section~\ref{Sec:Conclusion}.

\section{Related Work} \label{Sec:RelatedWork}

The randomized RANSAC method as proposed by Fischler and Bolles~\cite{fischler1981random} remains the most widely used consensus maximization algorithm til date. For a collection of several important geometric registration problems that have been embedded into RANSAC, the reader is kindly referred to the OpenGV framework by Kneip and Furgale~\cite{kneip2014opengv}. Though the focus of this work lies on deterministic CM, it is worth noting a few extensions to RANSAC that have been proposed over the years. Torr and Zissermann~\cite{torr00} propose MLESAC, a variation that sums up inlier likelihoods during the inlier verification stage rather than just a count based on a fixed inlier threshold. Nist\'er~\cite{nister03} introduces preemptive RANSAC, a variation that first generates many hypotheses and verifies them against a subset of the correspondences to eliminate wrong models. The remaining hypotheses are verified against a gradually larger subset of the correspondences while continuing to remove wrong models. Only very few hypotheses have to be verified against all data samples, and the strongest one survives. Chum et al.~\cite{chum03} propose LO-RANSAC, a variant that includes local optimization in each iteration, and thereby helps to validate the typical assumption that any outlier-free sample can lead to a model that agrees with all inliers. Later, Chum and Matas~\cite{chum05} propose PROSAC. Unlike RANSAC, samples are not treated equally but drawn from a progressively larger subset of the original correspondences, the latter being ordered by matching quality. Based on the assumption that higher matching scores lead to better correspondences, the algorithm potentially achieves orders of magnitude speed-up. More recently, Barath and Matas~\cite{barath2018graph} propose Graph-cut Ransac (GC-RANSAC). Similar to MLESAC, the method employs a probabilistic kernel to map distances onto the interval $[0,1]$ and sums up those values during the inlier verification stage. However, GC-Ransac additionally relies on the assumption of spatial proximity of inliers and outliers, and employs graph-cut for a globally optimal solution of the labelling problem.

While not addressing CM, it is worth listing early works using branch-and-bound for the globally optimal solution of geometric fitting problems as several important techniques for later deterministic CM approaches are gained from here. Olsson et al.~\cite{olsson06} and Hartley and Kahl~\cite{hartley07,hartley09} find geometrically globally optimal solutions to 2D-3D correspondence-based camera resectioning, a problem to which previous solvers exclusively employed algebraic error residuals. The solutions also go beyond previous globally optimal solvers as they are the first to enable guaranteed geometric optimality over the space of rotations. In particular, the method of Hartley and Kahl~\cite{hartley07,hartley09} is faster and introduces important bounding operations enabling branching over the space of rotations, a technique that proves valuable for many of the following spatial registration-based CM problems. By employing the $\ell_\infty$ norm, they furthermore propose the first geometrically globally optimal solution to the relative camera pose problem. Olsson et al.~\cite{olsson08} propose further branch-and-bound based geometrically globally optimal solutions to point, line, and plane based Euclidean 3D-3D registration as well as 3D-2D registration. The work introduces the technique of convex underestimators in combination with branch-and-bound, which later finds reuse in CM. The work claims that it simply needs convex sets to which points are registered, and thus hints at potential use in correspondence-less scenarios. However, in practice the convex sets consist of a single landmark, hence correspondence-free or outlier-affected scenarios are not yet addressed. Although slightly less related, it is worth mentioning that globally optimal solutions in the space of Euclidean transformations under an algebraic error criterion have later been derived using relaxations of Quadratically Constrained Quadratic Programs (QCQP)~\cite{briales18,zhao20}.

Back to CM, an interesting method able to digest a limited amount of outliers is presented by Kahl et al.~\cite{kahl2008l1}, who propose a global branch-and-bound based minimization of the L1-norm of the residual vector for triangulation and uncalibrated pose estimation. The branch-and-bound method remains the pre-dominant choice for deterministic variants. The seminal work of Li~\cite{li2009consensus} provides a general framework for solving problems that can be formulated using Direct Linear Transformation (e.g. line-fitting, essential matrix fitting). The work describes the problem as a bilinear program, and relaxes it to a convex underestimator in the form of a linear program by assuming bounds on the searched transformation. A globally optimal solution is then found by embedding the program into a branch-and-bound scheme. Later, Bazin et al.~\cite{bazin2012globally} introduce deterministic CM solutions to the problem of pure rotation estimation for panorama stitching, rotating lidar point set registration, and line clustering with vanishing point estimation. In their later work, they extend the pure rotation determination to the uncalibrated case~\cite{bazin2014}. An improved method for vanishing point estimation has recently been presented by Li et al.~\cite{li2020globally}. Yang et al.~\cite{yang2014} extend the DLT-based method of Li~\cite{li2009consensus} for robust, deterministic essential matrix fitting by branching over a minimal 5D essential matrix manifold parametrization. In their work, they reuse the rotation space branching method of Hartley and Kahl~\cite{hartley09}. The conceptually simpler robust correspondence-based point set registration is only introduced later~\cite{bustos2017guaranteed}. Speciale et al.~\cite{speciale2017consensus} introduce CM with linear inequality constraints.

More recently, we have seen a number of solutions for correspondence-based globally optimal vision-based relative pose estimation under a constrained motion model. Liu et al.~\cite{liu_globally_2021} address the case of known gravity, and Jiao et al.~\cite{jiao_deterministic_2021} utilize prior knowledge about the roll and pitch angles from a nearby reference pose. They furthermore decouple yaw and translation estimation, and employ a maximum clique estimator for the latter. Liu et al.~\cite{liu2022globally} propose solutions for planar ground vehicle motion, respectively. Jiao et al.~\cite{jiao2020globally} again employ translation invariant features and IMU readings in order to reduce the rotation search to a one-directional search. Above methods achieve a substantial speed-up by assuming prior knowledge and reducing the dimensionality of the branching space.

Deterministic CM has also been applied in situations in which data correspondences are not given upfront, but have to be found as a by-product of the optimization. An early seminal contribution is made by Breuel~\cite{breuel2003implementation}, who proposes solutions to globally optimal correspondence-less matching of points under a Euclidean transformation in the image plane. Later, by using inlier cardinality maximization and branching over 3D rotations, Yang et al.~\cite{yang2013,yang2016} solve globally optimally for pose and correspondence towards 3D point set registration. Their method---denoted Go-ICP---speeds up the estimation by adding local optimization to the lower bound calculation, thereby accelerating the pruning of branches inside BnB. Bustos et al.~\cite{parra2014fast} propose fast correspondence-less point set registration by creating possible match lists from stereo-graphic projections. Liu et al.~\cite{liu2018efficient} propose simplified globally-optimal point set registration by creating hypothetical match-lists from rotation invariant features. Similar highly efficient point-set registration methods have recently been proposed by Yang et al. \cite{yang2019polynomial,yang2020teaser}. Cai et al.~\cite{cai2019practical} address the globally optimal 4D point set registration problem given by using lidar scanners with built-in level compensation. A highly efficient method for 3D point-based rotation and correspondence is presented in \cite{peng2022arcs}, which is highly related to our method as it also makes use of interval stabbing in order to achieve a very substantial improvement in computational efficiency. Campbell et al.~\cite{campbell2017globally,campbell2020} finally propose a globally optimal solution to simultaneous camera pose and correspondence determination under a geometric error criterion, again branching over the space of rotations. Later, Campbell et al.~\cite{campbell2019} propose a related method in which a spherical 3D mixture-model describing the 3D spatial location of landmarks (e.g. objects) is aligned globally optimally with a 2D mixture-model describing the 2D image-based location of their projections (e.g. semantic segments). Hu and Kneip~\cite{hu2021} propose an interesting variation of the globally optimal correspondence-free point set registration problem in which they jointly solve for a symmetry plane, thereby enabling good performance in situations of extremely low overlap. Gao et al.~\cite{gao2020} again propose a method for globally optimal correspondence-free registration of points in the image plane under a constrained, non-holonomic transformation model. The method is successfully applied to a down-ward facing camera mounted on an Ackermann-steering vehicle for planar motion determination. Peng et al.~\cite{peng2020,peng2021} finally apply the globally optimal branch-and-bound technique to contrast maximization-based event camera motion estimation.

Whilst alternative techniques based on tree-search~\cite{chin_efficient_2015,cai_consensus_2019,cai_consensus_nodate} and bi-convex programming~\cite{cai2018deterministic,cai_consensus_nodate} exist, CM remains a hard problem~\cite{tat2020robust}. The present paper focuses on the predominant technique of branch-and-bound-based deterministic CM and the question how to increase its computational efficiency. However, in contrast to many existing solutions that employ problem-dependent techniques such as prior knowledge, algebraic manipulations, or invariant representations, our proposed method uses interval stabbing for a general speed-up of the original problem solution.

%%%%%%%%%%%%%%%%%%%%%%%%%%%%%%%%%%%%%%%%%%%%%%%%
%%%%%%%%%%%%%%%%%%%%%%%%%%%%%%%%%%%%%%%%%%%%%%%%
%%%%%%%%%%%%%%%%%%%%%%%%%%%%%%%%%%%%%%%%%%%%%%%%
%%%%%%%%%%%%%%%%%%%%%%%%%%%%%%%%%%%%%%%%%%%%%%%%
\section{Preliminaries}\label{Sec:BnB}

Before introducing our accelerated CM method, some preliminaries need to be covered: interval mapping, the plain branch and bound algorithm, and interval stabbing.
By incorporating these preliminaries, we aim to build a basic foundation for our findings.
When talking about the derivation in a branch and bound algorithm for the CM problem, interval mapping has always served as a useful tool. Dealing with intervals enables the relaxation of constraints.
The other two preliminary parts center on the branch and bound algorithm. The former provides the main algorithmic structure and the latter provides the core technique that enables searching in $1$-dim less space: interval stabbing.

\begin{definition}[Interval Mapping\cite{scholz2011deterministic}]\label{definition:interval-mapping}
    Let $X$ be an interval. Then the interval operation is defined by 
    \begin{equation}
        f(X):= \{f(x): x\in X \} = \big [ \min_{x\in X}\, f(x), \max_{x\in X}\, f (x)\big ],
    \end{equation}
    where $f: X \to \mathbb{R}$ denotes a continuous function such that $f(X)$ is an interval.
\end{definition}

\subsection{Branch and Bound}

As shown in~\cite{breuel2003implementation}, Branch and Bound (BnB) is a deterministic paradigm commonly used to maximise consensus and find the global optimum of geometric computer vision problems.
The main idea of BnB is to recursively branch over the solution space and calculate bounds for the maximum cardinality of the inlier subset on all candidate sub-regions,
 then prune cubes whose upper bound is smaller than the maximal lower bound so far.
 The algorithm keeps splitting cubes until the bounds are sharp enough, i.e. the lower bound is close to the upper bound.
Generally speaking, the optimality of BnB stems from the fact that it is an exhaustive search method that uses bounding operations to prune useless branches. However, BnB gains efficiency over a brute-force search by pruning useless branches (i.e. potentially large sub-regions) at an early stage without having to examine all sub-regions down to the leaf level.

Algorithm~\ref{Alg: BnB} presents a basic BnB pipeline using Best-First-Search. It has several ingredients: cube initialization (line~\ref{Alg1: CubeInitialization}), cube subdivision (line~\ref{Alg1: split}), lower bounding operation (lines~\ref{Alg1: lb2} and~\ref{Alg1: lb}), and upper bounding operation(lines~ \ref{Alg1: ub2} and~\ref{Alg1: ub}). Each of these steps is introduced in the following.

\noindent\textbf{Cube Initialization.} Cube Initialization (line~\ref{Alg1: CubeInitialization}, Algorithm~\ref{Alg: BnB}) refers to constructing an initial cube or hypercube $\Cube_0\subset \mathbb{R}^n$, over which the BnB algorithm searches for the global optimum of the consensus maximization objective~\eqref{Problem: CM}. $\Cube_0$ should therefore be chosen such that it actually contains the global optimum. Usually, this can be done fairly easily by considering the nature of the solved problem or the data distribution. For example, one can set $[-\pi, \pi]$ as the initial $1$-dimensional cube for an angular search.

\noindent\textbf{Cube Subdivision.} In each iteration (line~\ref{Alg1: WhileLoop}, Algorithm \ref{Alg: BnB}), BnB takes a cube or hypercube 
$$\Cube = [c_1^l, c_1^r] \times ... \times [c_n^l, c_n^r] \subset \mathcal{C} \subseteq \mathbb{R}^n , \label{eq: cube}$$ 
from the priority queue (line~\ref{Alg1: Pop}) and splits it into smaller ones (line ~\ref{Alg1: split}). The splitting rule is typically chosen such that a cube is divided into $2^n$ congruent sub-cubes. In practice, in order to balance accuracy and running time, additional stopping criteria are used such as limiting the maximal splitting depth $d$ or setting a tolerance $\tau$ that limits the minimal diameter of the sub-cubes.

\noindent\textbf{BnB - Lower Bound.}  
Given the current cube $\Cube$ taken from the priority queue (line~\ref{Alg1: Pop}), we define its centre point $\mathbf{b}_c$ as 
 $$\mathbf{b}_c = [(c_1^l+c_1^r)/2, ..., (c_n^l+ c_n^r)/2]^T.$$ 
Then, a trivial lower bound of~\eqref{Problem: CM} over $\Cube$ arises as the number of points that satisfy $0 \leq f(\mathbf{b}_c, \mathbf{s}_i)\leq \epsilon$,
 since for an arbitrary $\mathbf{b}$ and the global optimum $\mathbf{b}^*$ it  holds that 
 \begin{equation}
    \sum_{i=1}^N \; \mathbf{1}(f(\mathbf{b}, \mathbf{s}_i) < \epsilon ) 
    \leq \sum_{i=1}^N \; \mathbf{1}( f(\mathbf{b}^*, \mathbf{s}_i) < \epsilon ).
 \end{equation}
 This calculation of the lower bound happens in line~\ref{Alg1: lb}. Note that any point in the cube $\Cube$ defines a valid lower bound, and the centre point is chosen for the sake of convenience.

\begin{figure}[t] % notation of b is abused
    \renewcommand{\algorithmicrequire}{\textbf{Input:}}
    \renewcommand{\algorithmicensure}{\textbf{Output:}}
    \removelatexerror
    \begin{algorithm}[H]
        \caption{Branch and Bound}\label{Alg: BnB}
        \begin{algorithmic}[1] % The number tells where the line numbering should start
            \REQUIRE Dataset $\mathcal{S} = \{\mathbf{s}_i\}_{i = 1}^N $, threshold $\epsilon$
            \ENSURE Best solution $\bm{b}^*$
            % \STATE $\Cube_0 \gets$ $\mathtt{CubeInitialization}$ \label{Alg1: CubeInitialization}
            \STATE $\Cube_0 \gets$ CubeInitialization \label{Alg1: CubeInitialization}
            \STATE $q \gets$ initialize an empty priority queue
            \STATE $L^*, \bm{b}^* \gets$ getLowerBound($\Cube_0, \mathcal{S}, \epsilon$) \label{Alg1: lb2}
            \STATE $U(\Cube_0) \gets$ getUpperBound($\Cube_0, \mathcal{S}, \epsilon$) \label{Alg1: ub2}
            \STATE insert $\Cube_0$ into $q$ with priority $U(\Cube_0)$
            \WHILE{$q$ is not empty \label{Alg1: WhileLoop} }
                \STATE pop a cube $\Cube$ from $q$ with the largest upper bound \label{Alg1: Pop}
                \STATE $L(\Cube), \bm{b} \gets$ getLowerBound($\Cube, \mathcal{S}, \epsilon$) \label{Alg1: lb}
                % \STATE $U(\Cube) \gets$ getUpperBound($\Cube, \mathcal{S}, \epsilon$) \label{Alg1: ub}
                \IF {$L^* < L(\Cube)$}
                \STATE $L^* \gets L(\Cube)$
                \STATE $\bm{b}^* \gets \bm{b}$
                \ENDIF
                \IF {$U(\Cube) == L^*$} \label{Alg1: terminates}
                    \STATE \textbf{return} $\bm{b}^*$
                \ELSIF {$U(\Cube) > L^*$} 
                    \FOR{$\Cube_i$ {\bf in} splitCube($\Cube$) \label{Alg1: split} }
                        \STATE $U(\Cube_i) \gets$ getUpperBound($\Cube_i, \mathcal{S}, \epsilon$) \label{Alg1: ub}
                        \STATE insert $\Cube_i$ into $q$
                    \ENDFOR
                \ELSE
                    \STATE discard $\Cube$
                \ENDIF
            \ENDWHILE
            \STATE \textbf{return} $\bm{b}^*$
        \end{algorithmic}
    \end{algorithm}
\end{figure}
 
\noindent\textbf{BnB - Upper Bound.} An upper bound is typically obtained by relaxing the constraint.
Let $\Cube$ be a given cube and let $f$ be the constraint function for which the bounds of $f$ over $\Cube$ can be computed using interval mapping (Definition~\ref{definition:interval-mapping}). Note that, unlike usual, we are considering inputs of intervals which means the mapping of intervals is still an interval. Without loss of generality, denote the lower and upper bounds of $f$ at $\mathbf{s}_i$ over $\Cube$ as $f_i^l$ and $f_i^r$, that is
\begin{equation}\label{eq:fibound}
    f_i^l \leq f(\Cube, \mathbf{s}_i) \leq f_i^r.
\end{equation}
% Notice that the constraint in (\ref{Problem: CM}) equals to $-\epsilon < f(\mathbf{b}, \mathbf{s}_i) < \epsilon$.
Let us recall the constraint in (\ref{Problem: CM}) and the characteristics of residuals, i.e. $0 \leq f(\mathbf{b}, \mathbf{s}_i) < \epsilon$.
Then the situation when there is no valid solution $\mathbf{b}$ in cube $\Cube$ happens when $f_i^l > \epsilon$ or $f_i^r < 0$.
Therefore, through counting points that provide an interval that may contain a solution, we can get a valid upper bound for the summation as follows:
% Therefore, a possible upper bound on $\Cube$ could be obtained by counting the number of points where there is a valid solution that exists in $\Cube$.
\begin{equation}\label{eq:sum-fibound}
\begin{aligned}
    \max_{\mathbf{b}} \sum_{i}\mathbf{1}(f(\mathbf{b}, \mathbf{s}_i)<\epsilon) 
    \leq &\; \sum_{i}\mathbf{1}([f_i^l,f_i^r] \cap [0,\epsilon])\\
    = &\; \sum_{i}\mathbf{1}(f_i^l \leq \epsilon \text{ and } f_i^r \geq 0).
\end{aligned}
\end{equation}
The relaxation is achieved by the fact that valid solution intervals all contribute to the inlier count though in reality they may not intersect with each other.
Note that, once upper and lower bounds $f_i^l$ and $f_i^r$ have been derived for each $i$, the upper bound \eqref{eq:sum-fibound} can be computed in linear time.

\subsection{Interval Stabbing}

Suppose we are given a set of intervals $\mathcal{X} = \{[x_i^l, x_i^r]\}_{i = 1}^L$. 
The problem of interval stabbing aims to find the maximal subset of intervals $\mathcal{I}$
 that are intersecting with each other, or in a more figurative way, that could be stabbed by a stabber $s\in \mathbb{R}$.
 It can be formulated as follows
\begin{equation}
    \begin{aligned}
        \max_{s, \mathcal{I}} \; & \;|\mathcal{I}| \\
        \text{s.t.} \; & \; s\in [x_i^l, x_i^r], \\
        & \; \forall \;[x_i^l, x_i^r] \in \mathcal{I} \subseteq \mathcal{X}.
    \end{aligned}
    \label{Problem:IS}
\end{equation}
The problem of Interval Stabbing (IS) (\ref{Problem:IS}) can simply be solved in $\mathcal{O}(L\,logL)$ time and $\mathcal{O}(L)$ space as presented in algorithm~\ref{Alg: IS}. It was studied in \cite{berg1997computational} as a subproblem of windowing queries that can be deterministically and efficiently solved using advanced data structures, such as interval trees, segment trees and priority trees. % Notes: not 100% sure, I just quickly read the chapter. But the above arguments are basically the same as in GORE and ACRS
Adaptive voting generalized from histogram voting has a similar idea but its time complexity is $\mathcal{O}(L^2)$~\cite{yang2019polynomial,yang2020teaser,jiao2020globally,jiao_deterministic_2021}.
In recent practices, IS is found to be a powerful tool in optimization~\cite{bustos2017guaranteed,cai2019practical,li2020globally,peng2022arcs}.
In particular,~\cite{cai2019practical,li2020globally,peng2022arcs} used Interval Stabbing to efficiently solve an inner problem, an indicative example of how Interval Stabbing can help to accelerate algorithms.

\begin{figure}[!ht]
    \renewcommand{\algorithmicrequire}{\textbf{Input:}}
    \renewcommand{\algorithmicensure}{\textbf{Output:}}
    \removelatexerror
    \begin{algorithm}[H]
        \caption{Interval Stabbing}\label{Alg: IS}
        \begin{algorithmic}[1] % The number tells where the line numbering should start
            \REQUIRE Intervals $\mathcal{X} = \{[x_i^l, x_i^r]\}_{i = 1}^L $%, masks $\{[0,1] \}_{i = 1}^N$
            \ENSURE Best Stabber $s$, \#Stabbed intervals $nStabbed$
            \STATE $I \gets$ Sort all the endpoints in $\mathcal{X}$
            \STATE $n \gets 0, count \gets 0$
            \FOR{$i = 1$ \TO $2L$}
                % \IF {$M(i) == 0$}
                \IF {$I(i)$ is an left endpoints}
                    \STATE $count \gets count+1$
                    \IF {$count > n$}
                        \STATE $nStabbed \gets count$
                        \STATE $s \gets I(i)$
                    \ENDIF
                \ELSE
                    \STATE $count \gets count - 1$
                \ENDIF
            \ENDFOR
            \STATE \textbf{return} $s, nStabbed$ %\Comment{The gcd is b}
        \end{algorithmic}
    \end{algorithm}
\end{figure}

%%%%%%%%%%%%%%%%%%%%%%%%%%%%%%%%%%%%%%%%%%%%%%%%
%%%%%%%%%%%%%%%%%%%%%%%%%%%%%%%%%%%%%%%%%%%%%%%%
%%%%%%%%%%%%%%%%%%%%%%%%%%%%%%%%%%%%%%%%%%%%%%%%
%%%%%%%%%%%%%%%%%%%%%%%%%%%%%%%%%%%%%%%%%%%%%%%%
\section{Accelerated  Consensus Maximization}\label{Sec:ACM-X}

In this section, we present our method Accelerated Consensus Maximization (ACM).
Instead of searching in the original $n$-dimensional space, ACM searches and branches over an $n-1$ dimensional space, and uses IS to deal with the remaining variable.
Implementation-wise, ACM uses the same BnB diagram as shown in Algorithm~\ref{Alg: BnB}. 
However, different from BnB, in line~\ref{Alg1: CubeInitialization}, ACM initializes an $(n-1)$-dimensional cube. In line~\ref{Alg1: lb} and~\ref{Alg1: ub}, it devises the ACM bounding operations.
The main modification of ACM lies in adjusting the bounding operations to an $(n-1)$-dimensional cube.
In the following, we will first introduce the general idea of ACM
 and then show its global optimality. Finally, we analyze its time complexity.
 For a more intuitive understanding, we use plain BnB to refer to the standard BnB mentioned above.
 Note that the general idea here is to provide a basic intuition about ACM, details may differ for specific problems.

\subsection{Core of The Method} \label{sec:CoreMethod}

Consider the $1$-dimensional CM problem 
\begin{equation}
\begin{aligned}
    \max_{b\in \mathbb{R},\mathcal{I}} &\; |\mathcal{I} | \\
    \text{s.t.} &\; f(b,\mathbf{s}_i ) \leq \epsilon, \\
    &\; \forall \;\mathbf{s}_i \in \mathcal{I} \subseteq \mathcal{D}.
\end{aligned}
\label{prob: 1D CM}
\end{equation}
For each sample $\mathbf{s}_i$, the properly defined constraint function $f$ can be used to inversely find the union of all the possible intervals of $b$
\begin{equation}
    f(b,\mathbf{s}_i ) \leq \epsilon \; \Leftrightarrow \; b \in {\textstyle \bigcup\limits_j}  \;[b_{ij}^l,b_{ij}^r].
    \label{eq:inverse}
\end{equation}

\begin{figure}
    \centering
    \includegraphics[width = 0.8\linewidth]{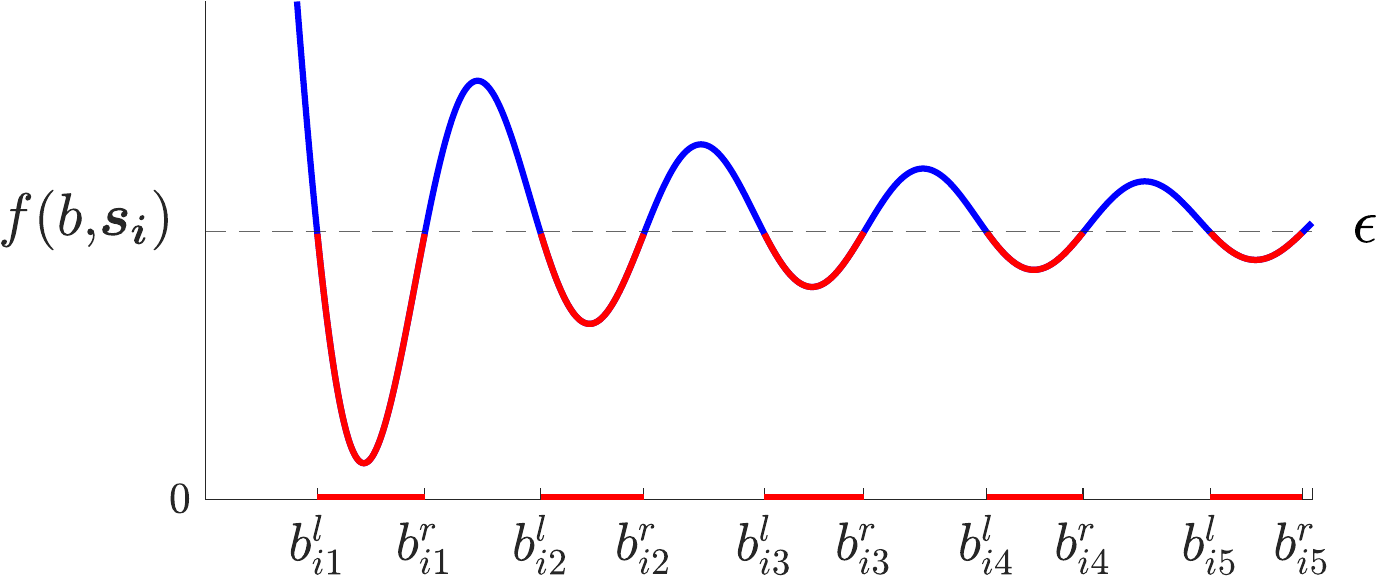}
    \caption{A visual example of (\ref{eq:inverse}): Given a threshold $\epsilon$ and sample $\mathbf{s}_i$, in general one can find a disjoint union of intervals (red, $x$-axis) for which every point $b$ satisfies $f(b,\mathbf{s}_i) \leq \epsilon$.} 
    \label{fig:eq8}
\end{figure}
\noindent A visual example is given in \figurename~\ref{fig:eq8}. Note that sample $\mathbf{s}_i$ is absent in the right-hand side of \eqref{eq:inverse}, as we use the index $i$ to denote the implicit dependency between $\mathbf{s}_i$ and the union of intervals ${\textstyle \bigcup_j}  \;[b_{ij}^l,b_{ij}^r]$. In other words, the precise form of $\;[b_{ij}^l,b_{ij}^r]$ depends on $\mathbf{s}_i$ and the constraint function $f$. While it is hard to make this dependency algorithmically explicit in our general discussion, we will see how to compute $\;[b_{ij}^l,b_{ij}^r]$ once $f$ and $\mathbf{s}_i$ are specified in concrete applications (Sections~\ref{Sec:1DProblem}-\ref{sec:4D}).

Merging the intervals if necessary, we could assume all intervals $[b_{ij}^l,b_{ij}^r]$ in \eqref{eq:inverse} are disjoint. With the equivalence in \eqref{eq:inverse}, every constraint $f(b,\mathbf{s}_i ) \leq \epsilon$ in (\ref{prob: 1D CM}) turns out to be some interval constraint, thus the 1-dimensional CM problem (\ref{prob: 1D CM}) reduces to an interval stabbing problem as in (\ref{Problem:IS}). In this way, we get the key observation: \textit{$1$-dimensional Consensus Maximization can be directly solved by Interval Stabbing},
 which finds a globally optimal solution without iteratively searching the solution space as plain BnB does. Note that, for the convenience of the discussion, in the following, we will suppose that the solution of $f(b,\mathbf{s}_i)<\epsilon$ will lead to only one interval instead of a union of intervals as illustrated in (\ref{eq:inverse}). Note that this notion-wise simplification does not affect the implementation. Having a union of intervals would simply mean that interval stabbing would have to be executed over additional intervals.
 
Let us now proceed by generalizing the above observation to the $n$-dimensional Consensus Maximization problem. 
Recall the general consensus maximization problem (\ref{Problem: CM}).
Plain BnB solves it by searching $\mathbf{b}$ over an $n$-dimensional space.
We propose Accelerated Consensus Maximisation (ACM) to search $n-1$ variables of $\mathbf{b}$ and use Interval Stabbing to solve for  the remaining $1$ variable.
As illustrated in \figurename~\ref{fig:VisualCubes}, benefitting from the lower-dimensional search space, ACM branches over a conceivably smaller number of sub-cubes which makes it converge significantly faster than plain BnB. Denote $\mathbf{b} = [b_1, b_2, ..., b_n]^T$.
Without loss of generality, we assume ACM branches over the first $n-1$ variables in $\mathbf{b}$, and we denote them as $\mathbf{b}_{1:n-1}$. 
Let $\Cube_{plain}$ represent a plain BnB cube in $n$-dimensional space and $\Cube_{ACM}$ represent the corresponding ACM cube in $(n-1)$-dimensional space. As visualized in \figurename~\ref{fig:VisualCubes}, $\Cube_{ACM}$ can be seen as $\Cube_{plain}$ less the last dimension. Next, let us introduce a general function representation for $f(\mathbf{b}|s_i)$ in which $b_n$ is conveniently separated from $\mathbf{b}_{1:n-1}$. It is given by
\begin{align}
    f(\mathbf{b}|\mathbf{s}_i) = \sum_{j}  f_j(b_{n}, h_{j}(\mathbf{b}_{1:n-1}) | \mathbf{s}_i),
\end{align}
where both $f_j$ and $h_j$ depend on the samples $\mathbf{s}_i$. 

\begin{figure}
    \centering
    \includegraphics[width = .8\linewidth]{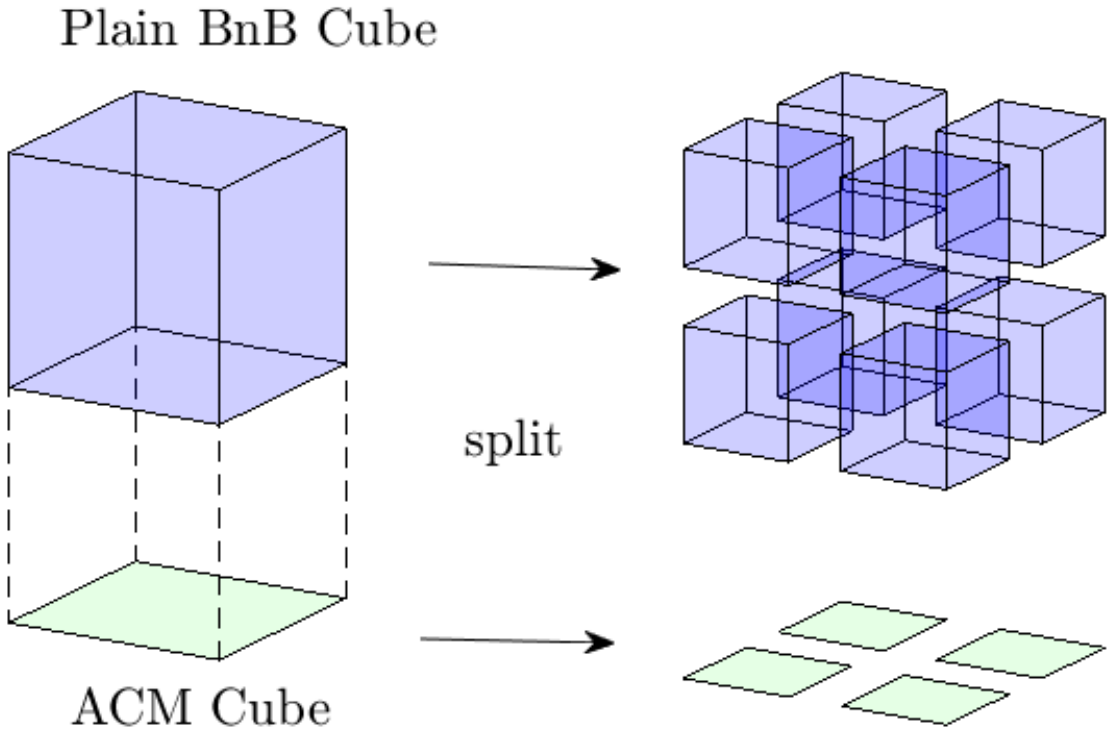}
    \caption{Search spaces of plain BnB (blue, top row) and ACM (green, bottom row). The left column shows example volumes and the right column the corresponding subcubes after splitting. Plain BnB splits an $n$-dimensional cube  into $2^n$ sub-cubes. ACM splits an $(n-1)$-dimensional cube into $2^{n-1}$ sub-cubes.}
    \label{fig:VisualCubes}
\end{figure}

To implement the above intuition of ACM, we need to answer two questions:

\begin{itemize}
    \item (\textit{ACM - Union}) Once $\mathbf{b}_{1:n-1}$ is given, how to \textit{invert} the constraint function $f$ to obtain the interval(s) $[b_{ni}^l,b_{ni}^r]$ in analogy to \eqref{eq:inverse}?
    \item (\textit{ACM - Bounds}) How to derive the lower and upper bounds of the objective for ACM (recall that ACM is a BnB-based method)?
\end{itemize}

For (\textit{ACM - Union}), we could in principle find all roots of $\left(\sum_{j}  f_j(b_{n}, h_{j}(\mathbf{b}_{1:n-1}) | \mathbf{s}_i)\right) - \epsilon = 0$ over $b_n$ using numerical methods and then recover the intervals $[b_{ni}^l,b_{ni}^r]$ from here. However, in most practical scenarios, we can solve this equation in closed form. Such closed form is problem-specific, so we will not discuss it at this moment and we will instead give concrete examples later.

Computing (\textit{ACM - Bounds}) is slightly more complicated. However, as we will next show, mild assumptions on the form of the function $f(\mathbf{b}|\mathbf{s}_i)$ will again enable us to reuse Interval Stabbing and solve for one dimension optimally while branching over the remaining $n-1$ variables.

\noindent\textbf{ACM - Lower Bound.} As explained above, for any $\mathbf{b}_{1:n-1}\in \Cube_{ACM}$ given, we can manage to find the interval of $b_n$ on each data sample $\mathbf{s}_i$
\begin{equation}
    f(b_n | \mathbf{b}_{1:n-1},\mathbf{s}_i) \leq\epsilon \; \Rightarrow \; b_n \in [\underline{b}_{ni}^l, \underline{b}_{ni}^r]%[b_{ni}^l, b_{ni}^r].
\end{equation}
Therefore, a valid lower bound can be found by doing Interval Stabbing on $\{ [\underline{b}_{ni}^l, \underline{b}_{ni}^r] \}_{i=1}^M$ %$\{[b_{ni}^l, b_{ni}^r]\}_{i = 1}^M$
\begin{equation}
\begin{aligned}
    L(\Cube_{ACM}) := &\; \max_{b_n} \sum_i \mathbf{1} (f(b_n | \mathbf{b}_{1:n-1},\mathbf{s}_i) \leq \epsilon)\\
    = &\; \max_{b_n} \sum_i \mathbf{1} (b_n\in [\underline{b}_{ni}^l, \underline{b}_{ni}^r]) \\
    \leq &\; \max_{\mathbf{b}} \sum_i \mathbf{1} (f(\mathbf{b},\mathbf{s}_i) \leq \epsilon),
\end{aligned}
\end{equation}
where the inequality here is derived by the fact that a randomly chosen $\mathbf{b}_{1:n-1}$ will always be worse or equal to the best solution $\mathbf{b}^*_{1:n-1}$.
Note that, interestingly, for the same problem the lower bound of ACM is tighter than the trivial lower bound of plain BnB as visualized in \figurename~\ref{fig:BetterLB}:
\begin{equation}
\begin{aligned}
    L(\Cube_{ACM}) = &\;\max_{b_n} \sum_i \mathbf{1} (b_n\in [\underline{b}_{ni}^l, \underline{b}_{ni}^r]) \\
    \geq &\; \sum_i \mathbf{1} (f(\mathbf{b},\mathbf{s}_i) \leq \epsilon) = L(\Cube_{plain})
\end{aligned}
\end{equation}
for any $\mathbf{b}\in \Cube_{plain}$ (with $\mathbf{b}_{1:n-1}$ fixed).
\begin{figure}
    \centering
    \includegraphics[width = .8\linewidth]{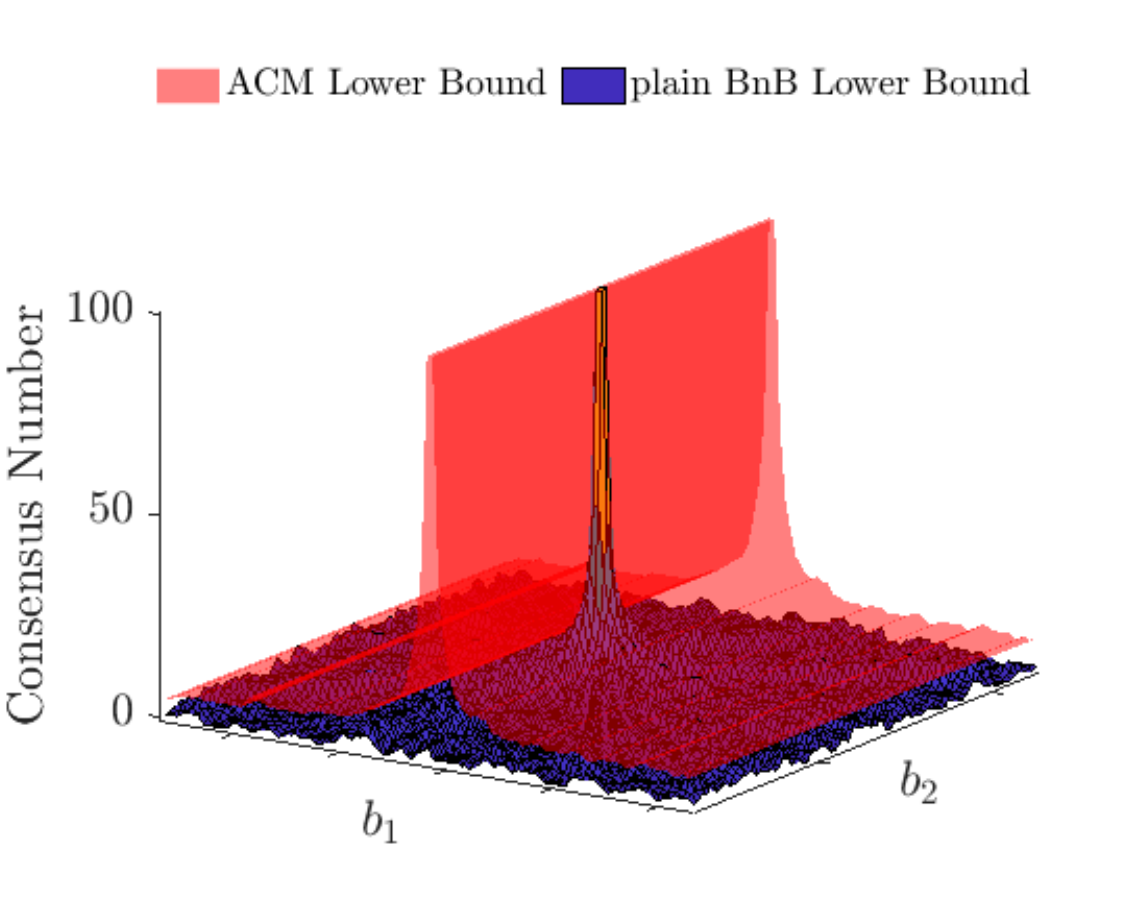}
    \vspace{-0.5cm}
    \caption{ Visualization of the lower bounds of plain BnB and ACM on a $2$-dimensional example. Plain BnB branches to search two variables $b_1$ and $b_2$. ACM searches over the space of $b_1$, and uses interval stabbing to determine $b_2$. The red curves are higher, i.e., the lower bounds of ACM are tighter.}
    \label{fig:BetterLB}
\end{figure}

\noindent\textbf{ACM - Upper Bound.} Next, let us assume that all $f_j(b_n,h_j(\mathbf{b}_{1:n-1})|\mathbf{s}_i)$ are monotonically increasing in  $h_j(\mathbf{b}_{1:n-1})$. This is a mild assumption as 1) monotonicity is only needed over the interval of $h_j(\mathbf{b}_{1:n-1})$ for $\mathbf{b}_{1:n-1}\in\mathcal{C}_{ACM}$, and that interval is decreasingly small as branching proceeds, and 2) the following derivations are equally simple to derive if some or all of the $f_j(b_n,h_j(\mathbf{b}_{1:n-1})|\mathbf{s}_i)$ are monotonically decreasing in $h_j(\mathbf{b}_{1:n-1})$, so in fact, monotonicity alone is a sufficient condition for the derivation.

After properly bounding the internal functions $h_j$ on $\Cube_{ACM}$ using Interval Arithmetic, we can get 
\begin{equation}
    h_{ij}^l \leq h_{j}(\Cube_{ACM}|\mathbf{s}_i) \leq h_{ij}^r,
\end{equation}
for each sample $\mathbf{s}_i$. Then, on $\Cube_{ACM}$ the general constraint (\ref{Problem: CM}) along with the basic characteristics of residuals translates into
\begin{equation}
\begin{aligned}
   &\;0 \leq \sum_j f_j(b_n,h_{j}(\Cube_{ACM})|\mathbf{s}_i) < \epsilon\\
   \Leftrightarrow &\; \big[\sum_j f_{j}(b_n,h_{ij}^l|\mathbf{s}_i), \sum_j f_{j}(b_n,h_{ij}^r|\mathbf{s}_i)\big] \cap [0,\epsilon].
   \label{eq:generalInterval1}
\end{aligned}    
\end{equation}

Again considering the first $n-1$ variables as given, interval (\ref{eq:generalInterval1}) can be relaxed as follows
\begin{equation}
    \begin{cases}
         \sum_j f_j(b_n|h_{ij}^l,\mathbf{s}_i) \leq \epsilon \\ \sum_j f_j(b_n|h_{ij}^r,\mathbf{s}_i) \geq 0,
    \end{cases}
\end{equation}
and using either numerical or closed-form analytical solutions, these inequalities can again be resolved to identify valid intervals $b_n\in\big[\overline{b}_{ni}^{l},\overline{b}_{ni}^r\big]$ for each sample $\mathbf{s}_i$.
A valid upper bound of ACM can be found by doing Interval Stabbing on these intervals, thus resulting in the following upper bound
\begin{equation}
    \begin{aligned}
        U(\Cube_{ACM}) 
        := &\; \max_{b_n} \, \sum_i \mathbf{1} (b_n\in \big[\overline{b}_{ni}^{l},\overline{b}_{ni}^r\big] ) \\
        \geq &\; \max_{\mathbf{b}} \sum_i \mathbf{1} (f(\mathbf{b},\mathbf{s}_i)  \leq \epsilon).
    \end{aligned}
\end{equation}

Again, the requirement of monotonicity is not a severe limitation. Monotonicity is only required on the sub-region on which the current bounds are evaluated. In practice, it is very often possible to achieve at least piece-wise monotonicity even for more complicated, non-linear functions.

For better illustration, we provide two kinds of constraint function $f$ and introduce how to separate $b_n$ out.

\begin{example} \label{example: add}
Suppose $f(\mathbf{b},\mathbf{s}_i)$ can be decomposed as the addition of two parts as
\begin{equation}
    f(\mathbf{b},\mathbf{s}_i) = h_1(\mathbf{b}_{(1:n-1)}|\mathbf{s}_i) + g_2(b_n|\mathbf{s}_i).
\end{equation}

\noindent For any $\Cube_{ACM}$ given, we can manage to find the interval of $h_1(\Cube_{ACM}|\mathbf{s}_i)$
\begin{equation}
        h_{i1}^l \leq h_1(\Cube_{ACM}|\mathbf{s}_i) \leq h_{i1}^r.
\end{equation}
Then, as discussed for the common upper bound of plain BnB, we can relax the objective of CM by counting the possible intervals as follows

\begin{equation}
\begin{aligned}    
    &\; [h_{i1}^l + g_{2}(b_n|\mathbf{s}_i), h_{i1}^r + g_2(b_n|\mathbf{s}_i)] \cap [0,\epsilon] \\
    \Leftrightarrow &\; g_2(b_n|\mathbf{s}_i) \in [-h_{i1}^r,\epsilon-h_{i1}^l] \\
    \Rightarrow &\; b_n \in [\overline{b}_{ni}^l, \overline{b}_{ni}^r],
\end{aligned}
\end{equation}
where we can deduce an interval of $b_n$ on $\Cube_{ACM}$ for each data sample $\mathbf{s}_i$. 
Therefore, a valid upper bound can be found by doing Interval Stabbing on $\{ [\overline{b}_{ni}^l, \overline{b}_{ni}^r] \}_{i=1}^M$, i.e.
\begin{equation}
        U(\Cube_{ACM}) := \max_{b_n} \, \sum_i \mathbf{1} (b_n \in [\overline{b}_{ni}^l, \overline{b}_{ni}^r]).
\end{equation}

\end{example}

\begin{example}
    Suppose $f(\mathbf{b},\mathbf{s}_i)$ can be decomposed as the multiplication of two parts as
\begin{equation}
    f(\mathbf{b},\mathbf{s}_i) = h_1(\mathbf{b}_{(1:n-1)}|\mathbf{s}_i) \cdot g_1(b_n|\mathbf{s}_i).
\end{equation}
\noindent For any $\Cube_{ACM}$ given, we can manage to find the interval of $h_1(\Cube_{ACM}|\mathbf{s}_i)$
\begin{equation}
        h_{1i}^l \leq h_1(\Cube_{ACM}|\mathbf{s}_i) \leq h_{1i}^r,
\end{equation}
and to find a finite range interval for $g_1$ we suppose that $h_{i1}^l > 0$. 
Without loss of generality, suppose $g_1(b_n|\mathbf{s}_i)$ is greater or equal to zero. Again, we can relax the objective of CM by counting the possible intervals as follows 
\begin{equation}
\begin{aligned}    
    &\; [h_{i1}^l \cdot g_{1}(b_n|\mathbf{s}_i), h_{i1}^r \cdot g_1(b_n|\mathbf{s}_i)] \cap [0,\epsilon] \\
    \Leftrightarrow &\; g_1(b_n|\mathbf{s}_i) \in [0,\epsilon/h_{i1}^l] \\
    \Rightarrow &\; b_n \in [\overline{b}_{ni}^l, \overline{b}_{ni}^r].
\end{aligned}
\end{equation}
A valid upper bound is again found by doing Interval Stabbing on $\{ [\overline{b}_{ni}^l, \overline{b}_{ni}^r] \}_{i=1}^M$.
\end{example}
Sections~\ref{Sec:1DProblem},~\ref{Sec:2DProblem}, and~\ref{Sec:3DProblem} will apply ACM to concrete geometric problems which will further explain its practical usage.

\subsection{Time Complexity Analysis}
Knowing that plain BnB and ACM are both globally optimal, let us proceed to compare their computational efficiency and analyse which one potentially converges faster to the global optimum. Note that, to compute upper and lower bounds at each iteration, plain BnB uses $\mathcal{O}(M)$ time, while ACM invokes the Interval Stabbing subroutine and therefore consumes $\mathcal{O}(M\log M)$ time. The extra $\log M$ factor appears as a slight disadvantage for ACM, but is in fact the cost that ACM pays for tremendous benefits: ACM is required to branch over a smaller space and yields tighter bounds, resulting in convergence within significantly fewer iterations. The following basic proposition mathematically consolidates our claim:

\begin{proposition}\label{Proposition: ACM iteration} 
    Run the plain BnB (resp. ACM) algorithm with the splitting rule that divides cubes into $2^n$ (resp. $2^{n-1}$) congruent sub-cubes and with the stopping criterion that terminates at the maximal splitting depth $d$. Then the following holds:
    \begin{itemize}[wide]
        \item (Perfect Bounds) If at every iteration both ACM and plain BnB prune all sub-cubes except one that contains the global optimum, then ACM needs $\mathcal{O}(2^{n-1})$ iterations to terminate, while plain BnB needs $\mathcal{O}(2^{n})$.
        \item (Invalid Bounds) If neither ACM nor plain BnB prunes any sub-cubes at each iteration, then ACM needs $\mathcal{O}(2^{(n-1)d})$ iterations to terminate, while plain BnB needs $\mathcal{O}(2^{nd})$. 
    \end{itemize}
\end{proposition}

Here we give a numerical feeling of Proposition~\ref{Proposition: ACM iteration}:
\begin{example}[\textit{Invalid Bounds}] Suppose the bounds of ACM and plain BnB are invalid in the sense of Proposition~\ref{Proposition: ACM iteration}. If $n = 3$ and $d = 10$, then ACM terminates in roughly $2^{20}$ iterations, and plain BnB in $2^{30}$. We furthermore have $2^{30}/ 2^{20}=1024$, implying that ACM can be $1024$ times faster than plain BnB (ignoring the extra $\log M$ factor of ACM due to Interval Stabbing).
\end{example}
Proposition \ref{Proposition: ACM iteration} suggests that if the upper and lower bounds for two algorithms are \textit{perfect}, ACM exhibits a minor advantage over plain BnB with a speedup factor of $2$, which will be compromised by the extra $\log M$ factor of invoking Interval Stabbing at every iteration. On the other hand, if the bounds are \textit{invalid} such that no cubes will be ruled out, then ACM will terminate exponentially faster than plain BnB, with a speedup factor of $2^d$; this advantage of ACM makes its extra $\log M$ term insignificant.

The scenarios of perfect and invalid bounds of Proposition \ref{Proposition: ACM iteration} are ideal; in practice, the bounds are neither perfect nor invalid. That said, ACM is preferred for two reasons: 
\begin{itemize}[wide]
    \item It is known that deriving tight lower and upper bounds is crucial but difficult for BnB-based methods, and many existing bounds tend to be ``invalid'' rather than ``perfect'' \cite{clausen1999branch,scholz2011deterministic}. Therefore, acceleration via ACM-based global optimization usually takes place.
    \item As shown in \figurename~\ref{fig:BetterLB}, the bounds we derive for ACM are based on Interval Stabbing and are more effective in pruning sub-optimal cubes than the usual bounds of plain BnB.
\end{itemize}

\figurename~\ref{fig:Bounds} illustrates the practical benefits of ACM. \figurename~\ref{fig:bounds11} shows the \textit{largest} lower bound of ACM is far tighter than plain BnB; this is a theoretical consequence of the lower bounding operation of ACM being tighter, as explained earlier (see \figurename~\ref{fig:BetterLB}). Interestingly, \figurename~\ref{fig:bounds11} shows the \textit{largest} upper bound of ACM is also tighter, even though the upper bounding operation of ACM is not necessarily better (\figurename ~\ref{fig:bounds12}). While this phenomenon seems contradictory at first glance, it is attributed to two factors: ACM has fewer sub-cubes as it branches over a smaller space; ACM prunes more sub-cubes due to its tighter lower bounds (\figurename~\ref{fig:bounds21}). In other words, ACM essentially calculates the largest upper bound over significantly fewer remaining sub-cubes (\figurename~\ref{fig:bounds22}), which is why its largest upper bound is also tighter.

\begin{figure}
    \centering
    \subfloat[Largest Bounds\label{fig:bounds11}]{%
       \includegraphics[height=0.37\linewidth]{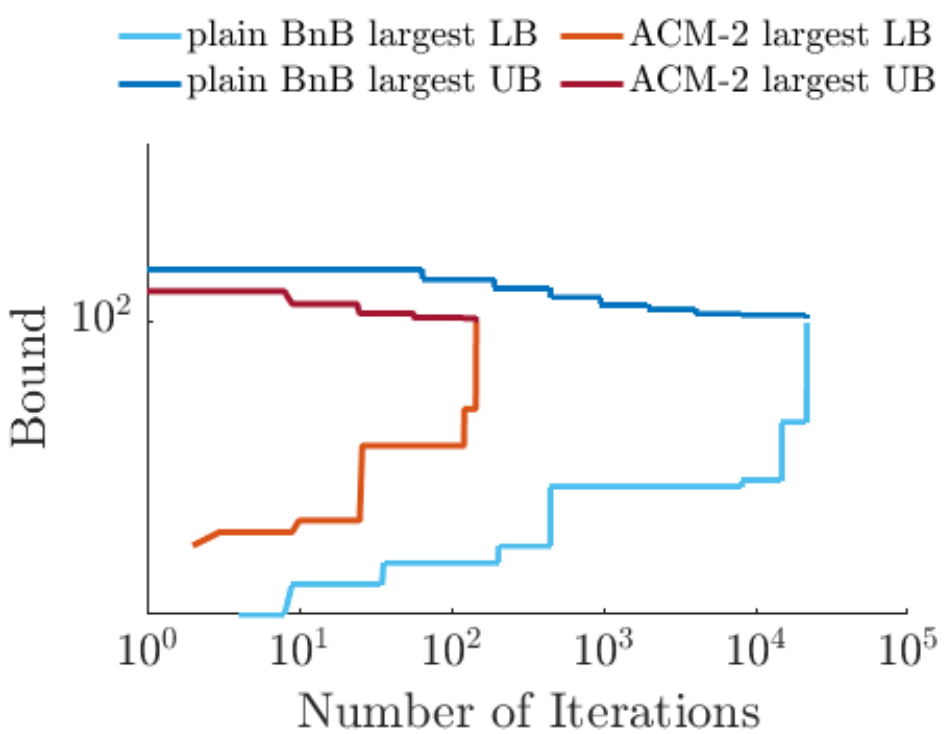}}
    \subfloat[Upper Bound\label{fig:bounds12}]{%
       \includegraphics[height=0.37\linewidth]{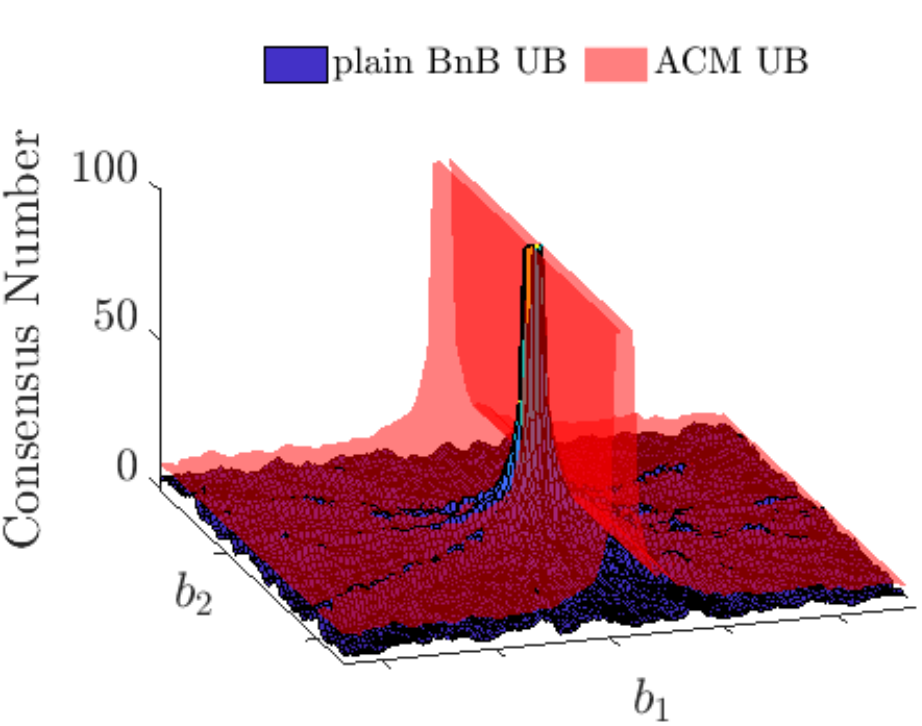}}
       
    \subfloat[Ratio of Pruned Cubes\label{fig:bounds21}]{%
       \includegraphics[height=0.35\linewidth]{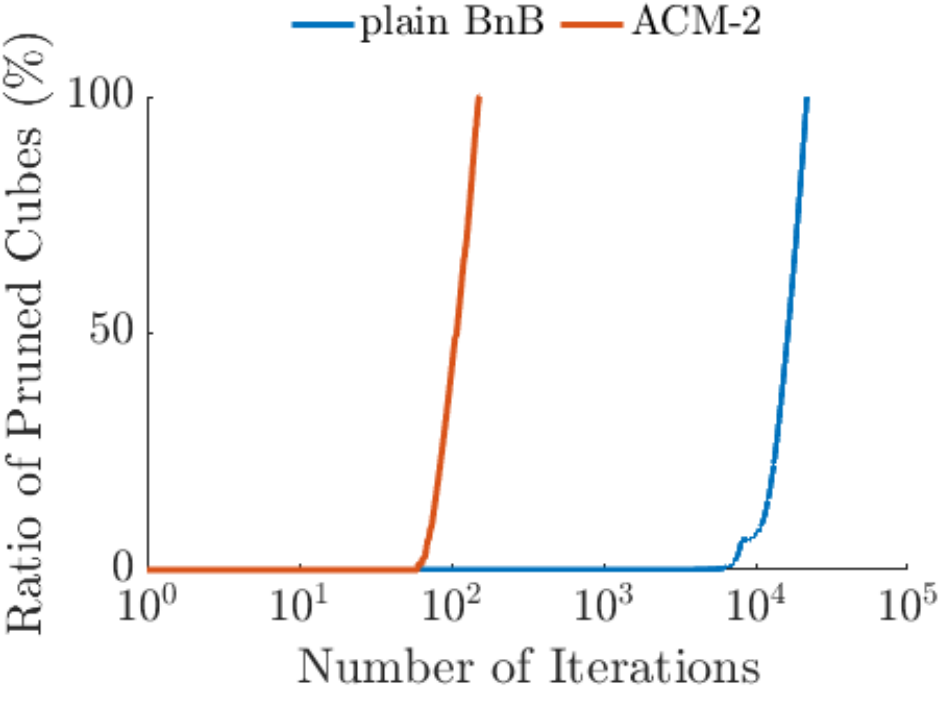}}       
    \subfloat[Remaining Cubes\label{fig:bounds22}]{%
       \includegraphics[height=0.35\linewidth]{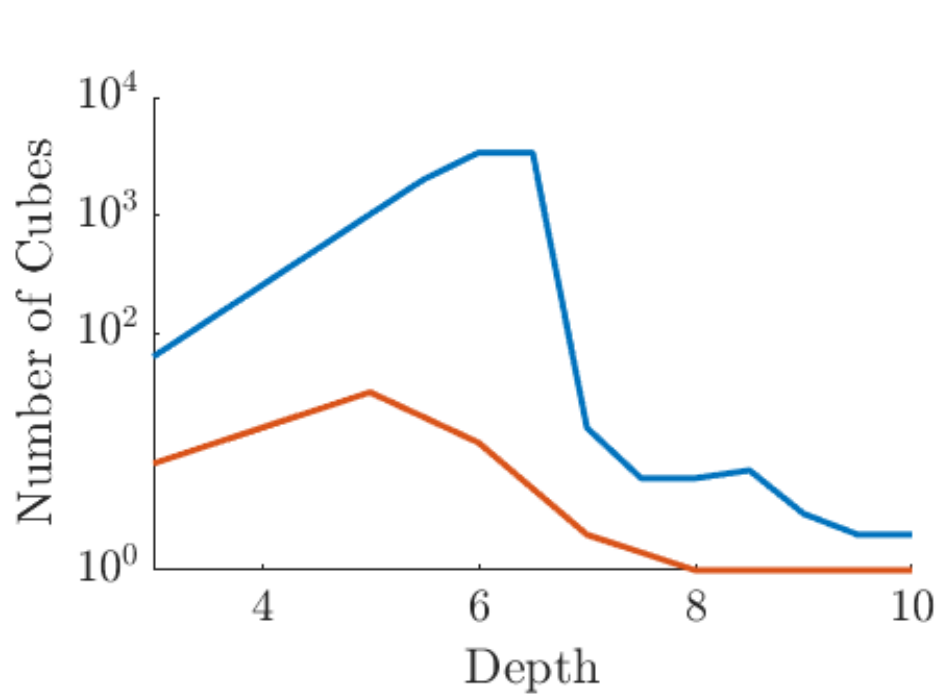}}
    \caption{Visual comparative analysis of ACM: The \textit{largest} upper and lower bounds of ACM are tighter (\ref{fig:bounds11}); the upper bounds of ACM and plain BnB are comparable (\ref{fig:bounds12}); ACM prunes a higher percentage of cubes in earlier stages (\ref{fig:bounds21}); ACM maintains much fewer cubes during execution (\ref{fig:bounds22}). }
    \label{fig:Bounds}
\end{figure}

\begin{remark} \label{Remark: termination}
    We use the maximal depth as a practical stopping criterion. Intuitively, setting a tolerance for the minimal diameter of sub-cubes seems more reasonable when comparing plain BnB and ACM. However, it can be shown that the same diameter tolerance approximately implies the same maximal depth. See the appendix for a more detailed explanation.
\end{remark}

\section{ACM-0: 3D-2D Registration}\label{Sec:1DProblem}
Firstly, we would like to introduce ACM concretely by solving a $1$-dimensional Consensus Maximization problem:
 Visual Inertial localization. This problem arises naturally in navigation systems,
 and aims at solving absolute camera pose from a set of 3D-2D correspondences given prior angular information provided by an Inertial Measurement Unit (IMU). Jiao et al.~\cite{jiao2020globally} propose a $1$-dimensional plain BnB algorithm to solve it globally optimally.
\\
\textbf{Problem formulation.}
Consider a set of 3D-2D correspondences $\{(\mathbf{p}_i, \mathbf{u}_i)\}_{i = 1}^M$,
 where $\mathbf{p}_i\in \mathbb{R}^3$ is a 3D point in world coordinates and $\mathbf{u}_i \in \mathbb{R}^2$ is a 2D image point. The correspondences follow an underlying absolute pose transformation constraint depending on $\mathbf{R}^*\in SO(3)$, $\mathbf{t}^* \in \mathbb{R}^3$, and a camera projection function $\pi: \mathbb{R}^3 \to \mathbb{R}^2$; explicitly: % It is given by
\begin{equation}
    \mathbf{u}_i = \pi(\mathbf{R}^* \mathbf{p}_i + \mathbf{t}^*)
    \label{eq:AbsolutePose}
\end{equation}
When camera calibration is known, $\mathbf{u}_i$ simply denotes the normalized image coordinates, and---using the homogeneous representation $\tilde{\mathbf{u}}_i = [\mathbf{u}_i^T,1]^T$--- (\ref{eq:AbsolutePose}) becomes the homogeneous equality constraint (omitting noise) 
\begin{equation}
    \tilde{\mathbf{u}}_i \propto \mathbf{R}^* \mathbf{p}_i + \mathbf{t}^*. \label{eq:AbsolutePose2}
\end{equation}
Using $\mathbf{R}^* = [\mathbf{r}_1^T, \mathbf{r}_2^T, \mathbf{r}_3^T]^T$, $\mathbf{t}^* = [t_1, t_2, t_3]^T$ and $\mathbf{u}_i = [u_{i1}, u_{i2}]^T$,
(\ref{eq:AbsolutePose2}) results in the two simple equality constraints
\begin{equation}
    \begin{cases}
        u_{i1} = \frac{\mathbf{r}_1^T\mathbf{p}_i + t_1}{\mathbf{r}_3^T\mathbf{p}_i + t_3} \\
        u_{i2} = \frac{\mathbf{r}_2^T\mathbf{p}_i + t_2}{\mathbf{r}_3^T\mathbf{p}_i + t_3}
    \end{cases}.
    \label{eq: AbsolutePose3}
\end{equation}
Representing the rotation $\mathbf{R}^*$ in angular form, we have
\begin{equation}
    \mathbf{R}^* = \mathbf{R}_z^*(\alpha^*)\mathbf{R}_y^*(\beta^*)\mathbf{R}_x^*(\gamma^*),
\end{equation}
where $\alpha^*, \beta^*, \gamma^*\in [-\pi, \pi]$ are the rotation angles along the $z$, $y$, and $x$ axis, respectively. Assuming $\beta$ and $\gamma$ can be estimated from inertial measurements (i.e. an IMU), only $4$ unknowns remain in the linear equations (\ref{eq: AbsolutePose3}).
Therefore, with $2$ pairs of correspondences $(\mathbf{p}_i, \mathbf{u}_i)$ and $(\mathbf{p}_j, \mathbf{u}_j)$, we can obtain $4$ equations to produce a unique solution. Using Translation Invariant Measurements (TIM) as proposed in~\cite{jiao2020globally}, 
 we can simplify the aforementioned $4$ equations to eliminate translation and arrive at a single equation that only relates the unknown angle $\alpha$
\begin{equation}
    d(\alpha^*)^{(ij)} := d_1^{(ij)} \sin \alpha^* + d_2^{(ij)} \cos \alpha^* + d_3^{(ij)} = 0, \label{eq:TIM}
\end{equation}
where $d_1^{(ij)} d_2^{(ij)}, d_3^{(ij)}$ are coefficients computed from correspondences.
As a result, the consensus maximization problem can be formulated as follows
\begin{equation}
    \begin{aligned}
        \max_{\alpha, \mathcal{I}} \;&\;|\mathcal{I}| \\
        s.t. &\; |d(\alpha)^{(ij)}| \leq \epsilon, \\
        &\; \forall \; \big ( (\mathbf{p}_i, \mathbf{u}_i), (\mathbf{p}_j, \mathbf{u}_j) \big) \in \mathcal{I}.    
    \end{aligned}
    \label{Problem: 1D}
\end{equation}
\noindent\textbf{Plain BnB \cite{jiao2020globally}.}
Denote the lower bound of $|d(\alpha)^{(ij)}|$ on a subcube $\Cube_{plain}$ as $L(\Cube_{plain})$ and the upper bound as $U(\Cube_{plain})$.
$L(\Cube_{plain})$ can be simply computed by the centre point of the cube.
On the cube $\Cube_{plain}$, we can obtain $d_l^{(ij)} \leq d(\alpha)^{(ij)} \leq d_r^{(ij)}$, thus the upper bound can be derived from the observation
\begin{equation}
    \begin{aligned}
        \max_{\alpha} \sum_{ij}\mathbf{1}( |d(\alpha)^{(ij)}| \leq \epsilon)
        & \, = \max_{\alpha} \sum_{ij} \mathbf{1}( -\epsilon  \leq d(\alpha)^{(ij)} \leq \epsilon)\\
        & \, \leq \sum_{ij}\mathbf{1}(d_l^{(ij)}  \leq \epsilon \text{ and } -\epsilon  \leq  d_r^{(ij)})\\
        & \, =: U(\Cube_{plain}),
    \end{aligned}
\end{equation}
where $\mathbf{1}(\cdot)$ is the indicator function.
\\
\textbf{ACM-0: Interval Stabbing.} Given that ACM can be used to replace $n$-dimensional plain BnB by branching over an $(n-1)$-dimensional space, only, we name our proposed solutions for specific $n$-dimensional problems as ACM-($n-1$). For example, in the case of the presently solved 1-dimensional CM problem, the method is denoted ACM-0.
For an inequality $|d_1^{(ij)} \sin \alpha + d_2^{(ij)} \cos \alpha + d_3^{(ij)}|<\epsilon$
 that is composed of basic functions like $sine$ and $cosine$, we can inversely solve for the interval of $\alpha$ with the assumption that $\alpha \in [-\pi, \pi]$.
See the appendix for more details about the angular interval resulting from constraints like (\ref{eq:TIM}).
Each data pair $\big( (\mathbf{p}_i, \mathbf{u}_i),(\mathbf{p}_j, \mathbf{u}_j) \big )$ then leads to an interval for $\alpha$ denoted $[\alpha_{ij}^l,\alpha_{ij}^r]$. 
Hence, interval stabbing can be used to directly solve for the optimal $\alpha$ without BnB as follows
\begin{equation}
\begin{aligned}
    \max_{\alpha} &\; \sum_{ij} \mathbf{1} (|d(\Cube_{ACM})^{ij}| \cap [0,\epsilon]) \\
    = \max_{\alpha} &\; \sum_{ij} \mathbf{1} (\alpha \in [\alpha_{ij}^l,\alpha_{ij}^r]).
\end{aligned}
\end{equation}

\subsection{Synthetic Experiments}
\label{sec:3d2dsynth}

\noindent\textbf{Setup.}
We present the following synthetic experiments to verify the advantage of ACM-0 over $1$-dimensional plain BnB.
To generate synthetic image correspondences, we randomly sample $200$ 3D points having a distance between $4$ and $8$ to the world frame origin,
and then generate the correspondences from rays that point to the same 3D point.
To add noise, we first extract an orthogonal plane from each bearing vector and add random noise in pixels by assuming a focal length of $800$ for the orthogonal plane. 
The perturbed points in the plane are renormalized to obtain perturbed bearing vectors~\cite{kneip2014opengv}.

Elements of translation $\mathbf{t}$ and the rotation angles are randomly generated in the interval $[-1,1]$ and $[-\pi/2, \pi/2]$, respectively. We furthermore add randomly sampled noise from $[-2,2]$ pixels for the image points, set $0.2$ as the inlier threshold for the algebraic constraint in (\ref{Problem: 1D}), and adopt $10$ as the maximal depth for plain BnB. All algorithms are implemented in C++ and repeated $100$ times on different random data to gain stable results.
The angular error is measured as $| \alpha_{gt} - \hat{\alpha} |$ in degrees, where $\alpha_{gt}$ and $\hat{\alpha}$ are the ground truth and the estimated angle, respectively.

\noindent\textbf{Results.}
As discussed, while plain BnB iteratively branches over intervals in 1 dimension, ACM-0 just directly uses Interval Stabbing, hence problem (\ref{Problem: 1D}) can be solved much more efficiently. In the following, we investigate the running time of plain BnB and ACM-0 and validate the performance of Interval Stabbing.
\figurename~\ref{fig:1Dtime} shows the mean results of plain BnB and ACM-0 where we vary the outlier ratio from $10\%$ to $90\%$ with a stepsize of $10\%$, and from $91\%$ to $95\%$ with a stepsize of $1\%$. \tablename~\ref{tab:ACM-0iter} presents the average number of iterations of plain BnB in large outlier ratio conditions. 

In general, ACM-0 is consistently faster than plain BnB and enjoys a $200 \times$-$300 \times$ speed-up ratio while leading to similar angular errors. More importantly, ACM-0 solves the $1$-dimensional problem globally optimally in real-time, while plain BnB does not attain real-time capabilities. 
More specifically, ACM-0 converges after about $10^{-2}$ sec even in extreme outlier cases (outlier ratio above $80\%$), while plain BnB needs over $1$ sec and more than $10^3$ iterations to complete. As furthermore shown in \figurename~\ref{fig:1Derror}, we obtain consistent angular errors with the method of Jiao et al.~\cite{jiao2020globally}.

\begin{figure}
    \centering
    \subfloat[Running Time\label{fig:1Dtime}]{%
       \includegraphics[height=.39\linewidth]{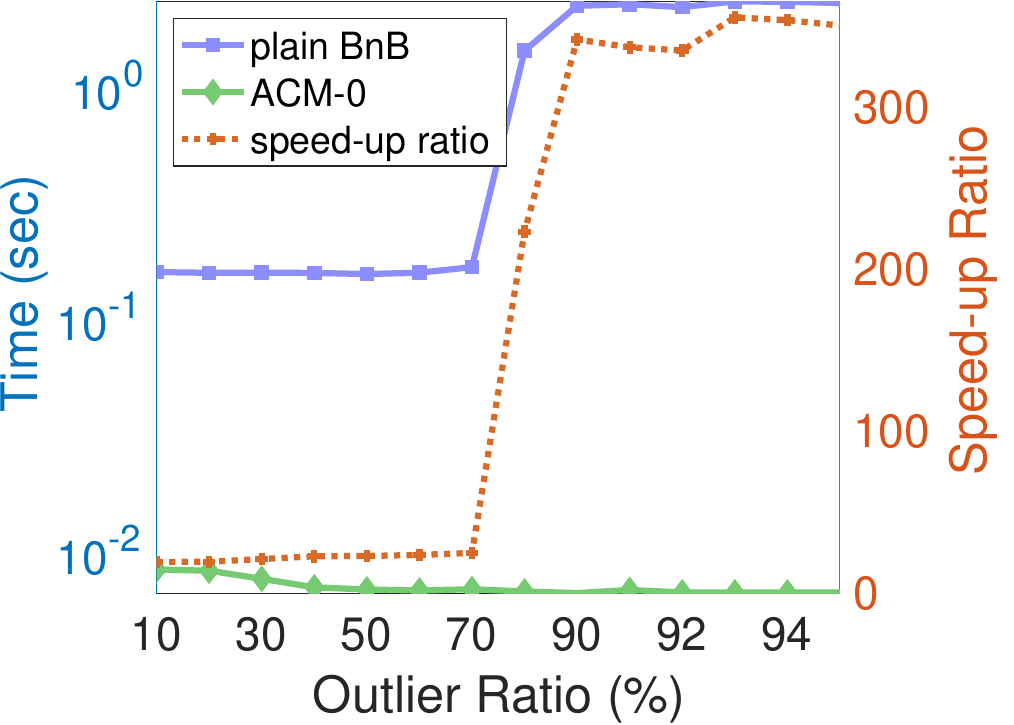}}
       \hfill
    \subfloat[Angular Error\label{fig:1Derror}]{%
       \includegraphics[height=.39\linewidth]{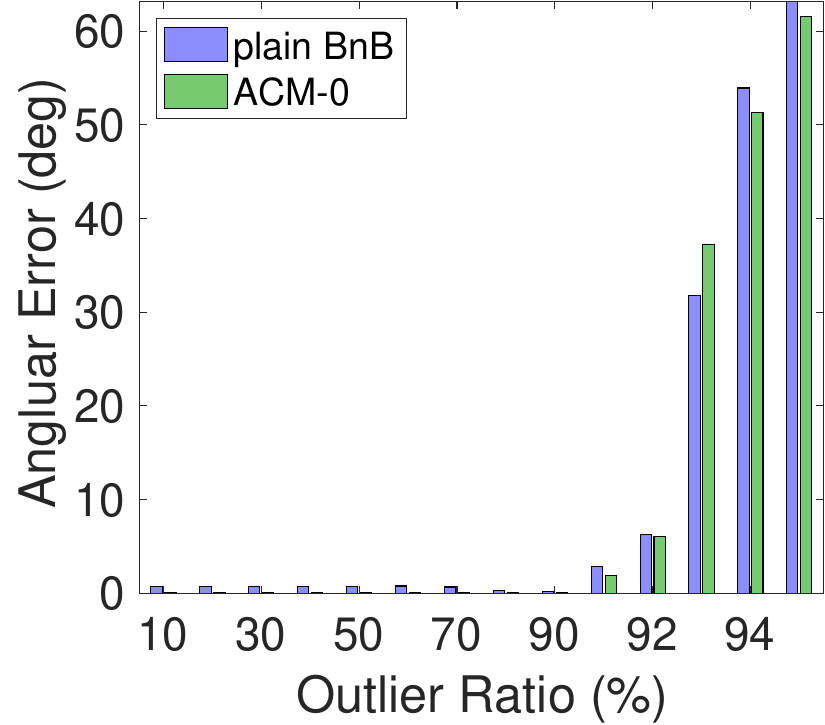}}
    \caption{3D-2D registration experiments (Section~\ref{Sec:1DProblem}) where we compare ACM-0 and the plain BnB method of \cite{jiao2020globally}. ACM-0 is more than $200$ times faster in the high-outlier regime (\figurename~\ref{fig:1Dtime}) with comparable angular errors (\figurename~\ref{fig:1Derror}). ($200$ randomly generated noisy correspondences, $100$ trials)  }
    \label{fig:1Dsimulation}
\end{figure}

\begin{table}[]
\caption{Number of iterations of plain BnB \cite{jiao2020globally}, which is always $1$ for ACM-0.}
\label{tab:ACM-0iter}
\centering
\begin{tabular}{c|rrrrrrrrrrrrrrr}
\toprule
ratio & 80\% & 90\% & 91\% & 92\% & 93\% & 94\% & 95\%\\
\midrule
iter & 1,138 & 1,842 & 1,861 & 1,822 & 1,920 & 1,898 & 1,899 \\
\bottomrule
\end{tabular}
\end{table}

\section{ACM-1: 2D-2D Registration}\label{Sec:2DProblem}
Next, we consider a $2$-dimensional Consensus Maximization problem and demonstrate how ACM can again achieve considerable acceleration by branching over a $1$-dimensional search space, only.
We consider the case of a forward-looking camera installed on a planar ground vehicle (e.g. a car driving on a flat road). The motion of the camera will be constrained to a plane, hence frame-to-frame visual odometry can be solved by finding a $3$-dimensional Euclidean transformation composed of a 2-dimensional translational and a $1$-dimensional rotational displacement in the plane. Note that this stands in contrast with five degrees of freedom in the general, calibrated relative pose scenario. By using an angle-plus-parallax parametrization and substituting in the epipolar constraint, the unobservable parallax is easily eliminated thereby resulting in a two-angle problem that can be solved globally optimally using plain BnB~\cite{liu2022globally}.

\begin{figure}
    \centering
    \includegraphics[width = .7\linewidth]{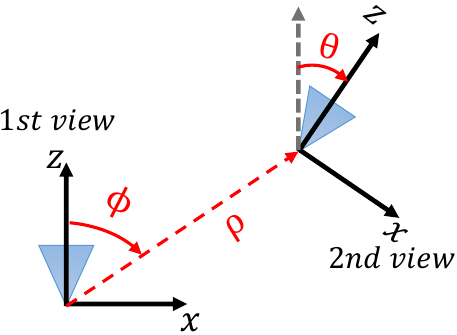}
    \vspace{-0.2cm}
    \caption{Visualization of the planar motion. The rotation is parametrized by an angle \textcolor{red}{$\theta$} as shown in  \eqref{equ:2d2dR}, and translation by its length \textcolor{red}{$\rho$} and angle \textcolor{red}{$\phi$} as shown in \eqref{equ:2d2dt}.}
    \label{fig:PlanarVisual}
\end{figure}
\noindent\textbf{Problem formulation.}
Consider a pair of $2$D-$2$D correspondences $(\mathbf{x}_i,\mathbf{x}_i^\prime)$ originating from the forward-looking camera on the vehicle under planar motion as illustrated in \figurename~\ref{fig:PlanarVisual}, where $\mathbf{x}_i = [u_1,v_1]^T\in \mathbb{R}^2$ and $\mathbf{x}_i^\prime = [u_2,v_2]^T\in \mathbb{R}^2$ denote the normalized image coordinates from the 1st view and 2nd view respectively. And denote the corresponding homogenous coordinates as $\tilde{\mathbf{x}}_i = [\mathbf{x}_i^T, 1]^T, \tilde{\mathbf{x}}_i^\prime = [\mathbf{x}_2^T, 1]^T$.
The correspondences follow an underlying relative pose transformation constraint depending on $\mathbf{R}^*\in SO(3)$ and $\mathbf{t}^* \in \mathbb{R}^3$. It is given by
\begin{equation}
    \mathbf{x}_i = \pi(\mathbf{R}^* \tilde{\mathbf{x}}_i^\prime + \mathbf{t}^*)
    \label{eq:RelativePose}
\end{equation}
where $\pi: \mathbb{R}^3 \to \mathbb{R}^2$ denotes the camera projection function.
Then the epipolar constraint is given by
\begin{equation}
    (\tilde{\mathbf{x}}_i)^T \mathbf{E}^* \tilde{\mathbf{x}}_i^\prime = 0, \label{eq:2D planar epipolar}
\end{equation}

and $\mathbf{E}^* = [\mathbf{t}^*]_\times\mathbf{R}^*$ is the essential matrix composed of camera rotation $\mathbf{R}^*\in SO(3)$ and translation $\mathbf{t}^*\in \mathbb{R}^3$ from the 2nd view to the 1st view.

Without loss of generality, we assume that the $y$-axis of the camera is pointing downwards and the principal axis forward. Given the planar motion of the vehicle in the horizontal plane, the rotation matrix $\mathbf{R}^*$ from the 2nd view to 1st view is given by the single-angle parameter matrix
\begin{equation}
    \mathbf{R}^* = \begin{bmatrix} 
    \cos\theta^* & 0 & \sin\theta^* \\
    0 & 1 & 0 \\
    -\sin\theta^* & 0 & \cos\theta^*
    \end{bmatrix},
    \label{equ:2d2dR}
\end{equation}
where $\theta^*$ is the yaw angle about the $y$ axis. 
Denoting $\rho^*$ the length of the translation vector and $\phi^*$ the direction angle of the translation with respect to the original position, the planar translation $\mathbf{t}^*$ is given by
\begin{equation}
    \mathbf{t}^* = \begin{bmatrix}
    \rho \sin\phi^* \\
    0 \\
    \rho \cos\phi^*
    \end{bmatrix}.
    \label{equ:2d2dt}
\end{equation}
See \figurename~\ref{fig:PlanarVisual} for a visualization of this setup.

Substituting $\mathbf{R}^*$ and $\mathbf{t}^*$ into (\ref{eq:2D planar epipolar}), we can easily obtain 
\begin{equation}
\begin{aligned}    
    &\;u_1v_2\cos\phi^* - v_2 \sin\phi^* \\
    &\;- u_2 v_1 \cos(\theta^* - \phi^*) - v_1\sin(\theta^* - \phi^*) = 0,\label{eq:constraint_Planar}
\end{aligned}
\end{equation}
where the length of the translation $\rho^*$ has been eliminated and only the $2$ unknown angles $\theta^*$ and $\phi^*$ are left to be found.
\\
Reformulating (\ref{eq:constraint_Planar}) in order to separate the unknown parameters, we can get
\begin{equation}
    A_1\sin(\theta_1^*+\phi_1^*) + A_2 \sin(\theta_2^*+\phi_2^*) = 0,
    \label{eq:constraint_Planar2}
\end{equation}
 where $\theta_1^* = \theta^* - \phi^*$, $\theta_2^* = \phi$ and $A_1, A_2, \phi_1$ and $\phi_2$ are coefficients computed from $\mathbf{x}_1$ and $\mathbf{x}_2$. Denoting $g_i(\theta_1, \theta_2) = A_1^{(i)}\sin(\theta_1+\phi_1^{(i)}) + A_2^{(i)} \sin(\theta_2+\phi_2^{(i)})$ as the above constraint for each correspondence $i = 1,...,M$, the Consensus Maximization problem can be naturally formulated as follows:
\begin{equation}
\begin{split}
     \max_{\theta_1, \theta_2, \mathcal{I}} \;&\;|\mathcal{I}| \\
    s.t. &\; |g_i(\theta_1, \theta_2)| \leq \epsilon, \forall \;(\mathbf{x}_i,\mathbf{x}_i^\prime) \in \mathcal{I}.
\end{split}
\label{Prob:2DPlanar}
\end{equation}
\\
\textbf{Plain BnB \cite{liu2022globally}.}
Note that the bounds generally described in Section~\ref{Sec:BnB} can be used to bound $|g_i(\theta_1, \theta_2)|$. Given a sub-cube $\Cube_{plain}$ with the centre point $[\theta_1^c,\theta_2^c]$ and $g^l_{i} \leq g_i(\Cube_{plain}) \leq g^r_{i}$, the lower bound $L(\Cube_{plain})$ and the upper bound $U(\Cube_{plain})$ are given by
\begin{align}
    L(\Cube_{plain}) = &\, \sum_i \; \mathbf{1} (|g_i(\theta_1^c,\theta_2^c)| \leq \epsilon)\\
    U(\Cube_{plain}) =&\, \sum_i \; \mathbf{1} (g_i^l \leq \epsilon \text{ and } g_i^r \geq -\epsilon).
\end{align}
\\
\textbf{ACM-1.}
Without loss of generality, suppose ACM-1 is used to branch over the space of $\theta_1$, and $\theta_2$ is searched by Interval Stabbing. First, consider the lower bound estimated by ACM-1. 
As mentioned for plain BnB, a trivial lower bound on a subcube $\Cube_{plain}$ can be obtained by the centre point $[\theta_1^c,\theta_2^c]$. Denoting the interval of $\theta_2$ inversely derived from $g_i(\theta_1^c, \theta_2)$ as 
$[\underline{\theta}_{2i}^l, \underline{\theta}_{2i}^r]$ (See the appendix for details), the lower bound for ACM-1 on a subcube $\Cube_{ACM}$ can be obtained by solving the Interval Stabbing problem
\begin{equation}
\begin{aligned}
    L(\Cube_{ACM-1}) = &\; \max_{\theta_2} \sum_i \mathbf{1} ( |g_i(\theta_1^c, \theta_2)|  \leq \epsilon) \\
    = &\; \max_{\theta_2} \sum_i \mathbf{1} (\theta_2 \in [\underline{\theta}_{2i}^l, \underline{\theta}_{2i}^r]).
\end{aligned}
\end{equation}
As introduced in Section~\ref{Sec:ACM-X}, the upper bound of ACM-1 can be obtained by solving the Interval Stabbing problem
\begin{equation}
\begin{aligned}
    U(\Cube_{ACM-1}) = &\; \max_{\theta_2} \sum_i \mathbf{1} ( |g_i(\Cube_{ACM-1}, \theta_2)| \cap [0,\epsilon]) \\
    = &\; \max_{\theta_2} \sum_i \mathbf{1} (\theta_2 \in [\overline{\theta}_{2i}^l, \overline{\theta}_{2i}^r]).
\end{aligned}
\end{equation}

\subsection{Synthetic Experiments}\label{subsection:2DProblem-synthetic-experiments}
\noindent\textbf{Setup.} 
We use the same setting with noise as mentioned in the last section except that now the camera undergoes planar rotation and translation to generate coordinates in a second view. We randomly sample the yaw angle within the interval $[-\pi/3, \pi/3]$ while the roll and pitch angles are kept to $0$. The translation angle is sampled from the interval $[-\pi/3, \pi/3]$, and the translation norm $\rho$ from the interval $[-2,2]$.  We still use $10$ as the maximal depth and $0.02$ as the inlier threshold for both plain BnB~\cite{liu2022globally} and ACM-1.

\begin{figure}
    \centering
    \subfloat[Running Time\label{fig:2Dtime}]{%
       \includegraphics[height=.39\linewidth]{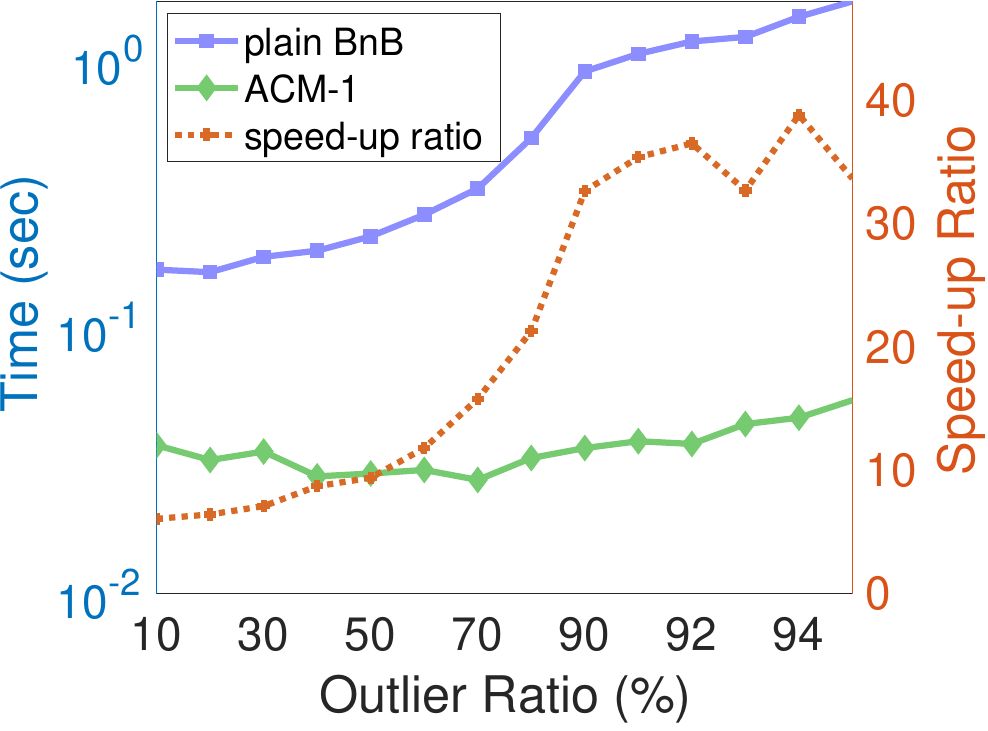}}
       \hfill
    \subfloat[Angular Error\label{fig:2Derror}]{%
       \includegraphics[height=.39\linewidth]{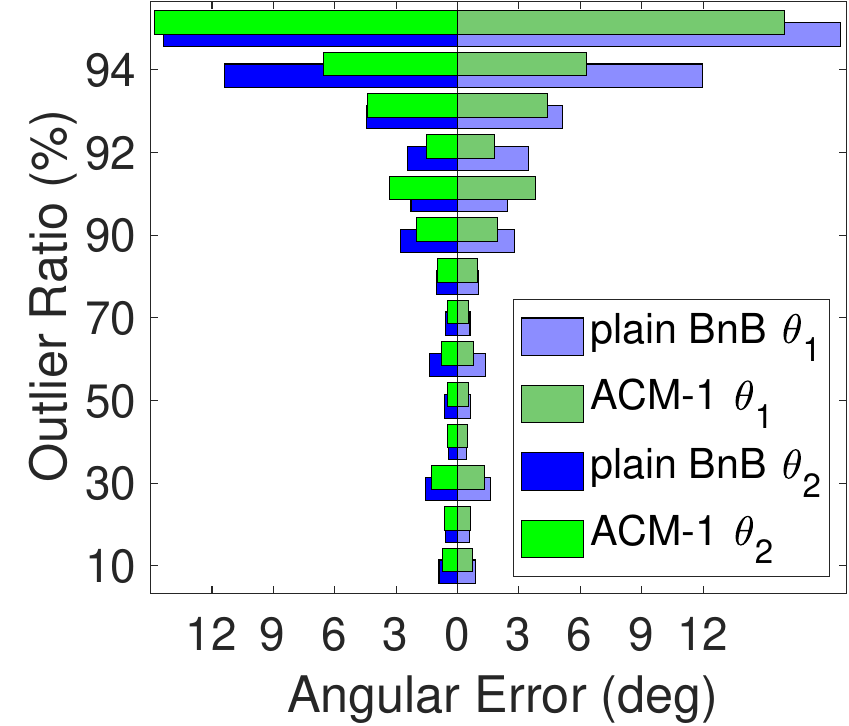}}
       
    \subfloat[Number of Iterations\label{fig:2Diter}]{%
       \includegraphics[height=.41\linewidth]{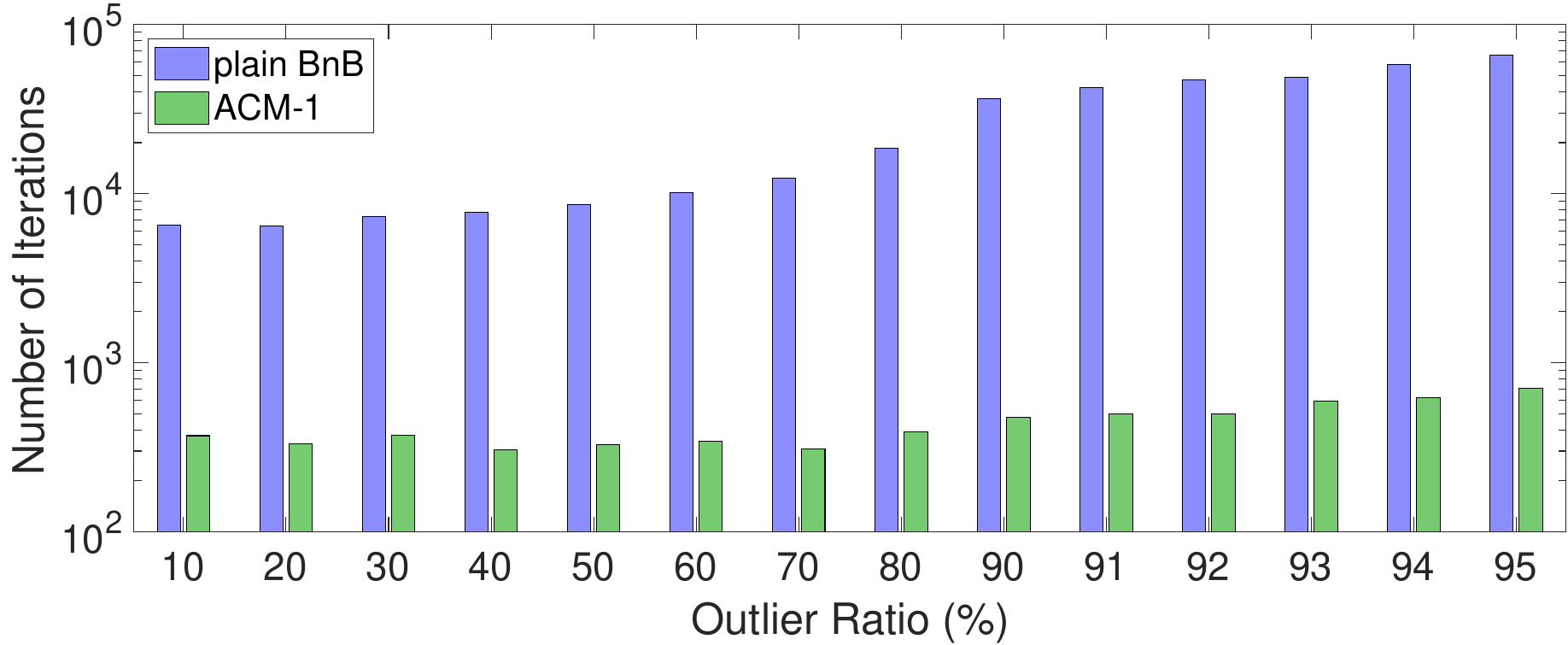}}
    \caption{2D-2D registration experiments (Section~\ref{subsection:2DProblem-synthetic-experiments}). ACM-1 attains $20\times$ speed-up over plain BnB at extreme outlier rates (\figurename~\ref{fig:2Dtime}) with comparable angular errors (\figurename~\ref{fig:2Derror}) and converging within much fewer iterations (\figurename~\ref{fig:2Diter}). ($200$ randomly generated noisy correspondences, $100$ trials) }
    \label{fig:2Dsimulation}
\end{figure}

\noindent\textbf{Results.} We implement both plain BnB~\cite{liu2022globally} and ACM-1 based on the bounding functions defined previously and investigate their relative performance. We were not able to run the code of~\cite{liu2022globally}, so we implemented our own version of plain BnB. In~\cite{liu2022globally} it was written that their implementation took $18$ seconds to process $50$ correspondences from two consecutive KITTI images with $10^{-4}$ inlier threshold. On the other hand, our implementation took fewer than $1$ second to process $200$ SIFT correspondences with the same inlier threshold $10^{-4}$. This validates the correctness and efficiency of our implementation of plain BnB as a baseline.

The outlier ratio is varied from $10\%$ to $90\%$ with a step size of $10\%$, and then from $91\%$ to $95\%$ with a step size of $1\%$. The mean results are displayed in \figurename~\ref{fig:2Dsimulation}.
As expected, it is clear that ACM-2 needs much less time to converge than plain BnB and achieves a $2 \times$-$20 \times$ speed-up ratio (\figurename~\ref{fig:2Dtime}).
At the same time, ACM-2 maintains similarly low angular errors than plain BnB for both $\theta_1$ and $\theta_2$ (\figurename~\ref{fig:2Derror}). 
Taking a detailed look at \figurename~\ref{fig:2Dtime} indicating run-times, ACM-2 stably takes an average of about $0.03$ sec for a wide range of outlier ratios. This is contrast to plain BnB, for which running times are much more sensitive with respect to the outlier ratio and increase from $0.0615$ sec to $0.5833$ sec.
Taking into account the significant difference in the number of iterations as shown in \figurename~\ref{fig:2Diter}, this result also confirms our claim that in practice the extra $\log M$ complexity for interval stabbing in each iteration is negligible.

\subsection{Real Experiments}\label{subsection:2DProblem-real-experiments}
In order to verify the effectiveness of our algorithm for globally optimal 2D-2D registration, we conducted real-world experiments using the KITTI dataset~\cite{geiger2013vision}. 

\noindent\textbf{Setup.}
We process all 11 sequences (00-10) containing ground truth. 
We extract SIFT features~\cite{lowe2004distinctive} for each image, set the number of best features to retain (\textit{nfeatures}) to $1000$, and perform brute-force matching of feature points on pairs of consecutive frames. Finally, the obtained pixel correspondences were normalized using the known intrinsic camera matrix. Ground truth relative pose parameters $\mathbf{R}$ and $\mathbf{t}$ are extracted by concatenating the corresponding absolute poses, followed by extraction of the ground truth values for $\theta_{gt}$ and $\phi_{gt}$.

 The threshold $\epsilon$ for the algebraic error is set to $10^{-4}$, and the maximal depth for plain BnB and ACM-1 are set to $20$ to obtain more precise solutions. 
The initial cubes of plain BnB and ACM-1 are defined as $[-\pi/2, \pi/2]^2$  and $[-\pi/2, \pi/2]$ respectively. Errors are expressed by the absolute difference between estimated and ground truth angles $\theta$ and $\phi$.

\noindent\textbf{Results.}
Results on the KITTI datasets again demonstrate the expected speed-up of ACM-1 over plain BnB.
The median results over all image pairs in each sequence are summarised in \tablename~\ref{tab: KITTI}.
Overall, ACM-1 obtains about a $2 \times$ speed-up and almost identical errors. The faster running time of ACM-1 is again the result of a significantly reduced number of iterations. Note that---although the SIFT feature is rather old---it still returns high-quality correspondences between the pairs of views, thereby leading to low outlier scenarios in which the advantage of ACM over plain BnB is less pronounced.

\begin{table*}[t!]
\centering
\caption{Experiments on the KITTI dataset (Section \ref{subsection:2DProblem-real-experiments}). Results are the median over all image pairs of each sequence. }
\label{tab: KITTI}
\begin{tabular}{cc|rrrrrrrrrrr}
\toprule
\textbf{Metric} & \textbf{Method $\backslash$ Seq.} & \textbf{00} & \textbf{01} & \textbf{02} & \textbf{03} & \textbf{04} & \textbf{05} & \textbf{06} & \textbf{07} & \textbf{08} & \textbf{09} & \textbf{10}\\
\midrule
\multirow{2}{*}{Rot (deg)} & plain BnB      & 0.1304 & 0.1052 & 0.0894 & 0.1053 & 0.0544 & 0.0730 & 0.0862 & 0.0635 & 0.0765 & 0.0802 & 0.0879 \\
& ACM-1  & 0.1293 & 0.1050 & 0.0891 & 0.1008 & 0.0536 & 0.0733 & 0.0872 & 0.0622 & 0.0767 & 0.0801 & 0.0852 \\
\multirow{2}{*}{Trans (deg)} & plain BnB      & 4.4905 & 2.6027 & 2.6733 & 5.2189 & 1.3994 & 3.5529 & 2.9525 & 4.7887 & 3.2316 & 2.9904 & 3.2244 \\
& ACM-1 & 4.4743 & 2.5903 & 2.6752 & 5.4160 & 1.4357 & 3.5397 & 2.9182 & 4.7979 & 3.2276 & 2.9550 & 3.1836 \\
\midrule
\multirow{2}{*}{Time (sec)} & plain BnB      & 1.0639 & 0.5721 & 0.9455 & 1.7012 & 0.8860 & 1.1202 & 0.8915 & 1.4539 & 1.0266 & 0.9447 & 1.0122 \\
& ACM-1 & 0.5041 & 0.3033 & 0.4697 & 0.7798 & 0.5364 & 0.5390 & 0.4561 & 0.7330 & 0.5158 & 0.4786 & 0.4871 \\
\multirow{2}{*}{Iter} & plain BnB      & 16,234 & 10,047 & 15,571 & 22,793 & 13,524 & 16,238 & 14,043 & 19,926 & 16,107 & 15,198 & 15,951 \\
& ACM-1 & 1,692  & 1,307  & 1,712  & 2,197  & 1,713  & 1,750  & 1,664  & 2,095  & 1,792  & 1,711  & 1,714  \\
\midrule
\multicolumn{2}{c|}{Speed Up Ratio} & 2.0116 & 1.8783 & 1.9562 & 2.1045 & 1.5778 & 1.9840 & 1.9061 & 1.9645 & 1.8939 & 1.9430 & 1.9971 \\
\bottomrule
\end{tabular}
\end{table*}

\begin{remark}
    There might be a concern about algebraic constraints used in geometric vision problems of Secs.~\ref{Sec:1DProblem} and~\ref{Sec:2DProblem} leading to less accurate results than enforcing geometric constraints. We would like to point out that the use of algebraic errors in consensus maximization or robust estimation problems is common, and, as demonstrated in prior works, algebraic error metrics are in most cases powerful enough to at least classify outliers. It would furthermore be straightforward to append geometric error-based refinement to the solution identified by ACM.
\end{remark}

\section{ACM-2: 3D-3D Registration}\label{Sec:3DProblem}

In this section, we apply ACM to the problem of 3D point set registration (i.e. Procrustes alignment), which aims at aligning two point sets by a rigid transformation. % with the scale known

\noindent \textbf{Problem formulation}
Consider two overlapping noiseless 3D point sets $\mathcal{P} = \{\mathbf{p}_i\}_{i=1}^{M_1}$ and $\mathcal{Q} = \{\mathbf{q}_j\}_{j=1}^{M_2}$. Each overlapping pair $(\mathbf{p}_i, \mathbf{q}_j)$ satisfies the relation 
\begin{equation}
    \mathbf{q}_j = \mathbf{R}^* (\mathbf{p}_i + \mathbf{t^*}),
    \label{eq:CorrlessRigid}
 \end{equation}
where $\mathbf{R}^*\in SO(3)$ and $\mathbf{t^*}\in \mathbb{R}^3$ represent the underlying rotation and translation respectively.

Instead of solving the $6$-DoF problem directly, it has been shown that algebraic operations can be used to eliminate $3$ DoF and solve globally for the resulting $3$-DoF problem by Branch and Bound.
Furthermore, ACM can reduce $1$ more DoF to solve it by branching over a $2$-dimensional search space. 

Similar to the problem of image registration, it is possible to extract 3D features on point sets and match correspondences,
which leads to correspondence-based 3D-3D registration. On the other hand, it is also possible to simultaneously search for transformation and correspondence, thereby leading to so-called correspondence-less registration.
In the following, we present ACM-2 for both variants of 3D point set registration. 
We name ACM for the correspondence-based 3D-3D registration problem \textit{ACM-2 Corr}, and ACM for correspondence-less 3D-3D registration \textit{ACM-2 Corrless}.

\subsection{Correspondence-based Registration}\label{Sec:3DCorrProblem}
Provided a set of 3D point correspondences $\{(\mathbf{p}_i, \mathbf{q}_i) \}_{i = 1}^M$, where $\mathbf{p}_i, \mathbf{q}_i \in \mathbb{R}^3$, the problem of correspondence-based point set registration can be understood as finding the rotation $\mathbf{R}^*$ and translation $\mathbf{t^*}$ that satisfy the relation
\begin{equation}
    \mathbf{q}_i = \mathbf{R}^* (\mathbf{p}_i + \mathbf{t^*}), % + \epsilon_i, 
    \label{eq: 3DCorrbased}
\end{equation}
where we omitted noise for simplicity.  Taking the norm on both sides, rotation invariance easily leads to the constraint
\begin{equation}
    \| \mathbf{q}_i \|_2 = \| \mathbf{p}_i + \mathbf{t}^* \|_2. \label{eq: 3DCorrbased RI}
\end{equation}
The unknown rotation $\mathbf{R}^*$ is thus eliminated, and the unknown translation $\mathbf{t}^*$ can be solved by the following Consensus Maximization problem 
\begin{equation}
    \begin{aligned}
        \max_{\mathbf{t}, \mathcal{I}} &\; |\mathcal{I}| \\
        s.t. &\; |\| \mathbf{q}_i \|_2 - \| \mathbf{p}_i + \mathbf{t} \|_2 | \leq \epsilon, \forall \;(\mathbf{p}_i, \mathbf{q}_i) \in \mathcal{I}
    \end{aligned}
    \label{Problem: 3DCorrbased}
\end{equation}
\textbf{Plain BnB.}
Let us define $h_i(\mathbf{t}) = |\| \mathbf{q}_i \|_2 - \| \mathbf{p}_i + \mathbf{t} \|_2 |$.
Given a subcube $\Cube_{plain}$ with the centre point $\mathbf{t}_c$, we may denote the interval of $h_i(\mathbf{t})$ on this volume as $[h_i^l, h_i^r]$. Again, a trivial lower bound and the intuitive upper bound result to
\begin{align}
    &L(\Cube_{plain}) = \sum_i \; \mathbf{1} (h_i(\mathbf{t}_c) \leq \epsilon)\\
    &U(\Cube_{plain}) = \sum_i \; \mathbf{1} (h_i^l \leq \epsilon \text{ and } h_i^r \geq -\epsilon).
\end{align}
\\
\textbf{ACM-2 Corr.} Write the translation $\mathbf{t}$ and point $\mathbf{p}_i$ respectively as $\mathbf{t} = [t_1,t_2,t_3]^T$ and $\mathbf{p}_i = [p_{1i},p_{2i},p_{3i}]^T$. Without loss of generality, let ACM-2 branch over the first two dimensions of the translation space and solve globally optimally for $t_3$ using Interval Stabbing.

For the constraint $h_i(\mathbf{t}) \leq \epsilon$, we can derive
\begin{align}
    &\; \qquad -\epsilon \leq \| \mathbf{q}_i \|_2 - \| \mathbf{p}_i + \mathbf{t} \|_2 \leq \epsilon \\
    \Leftrightarrow &\; \quad \;\| \mathbf{q}_i \|_2 -\epsilon \leq \| \mathbf{p}_i + \mathbf{t} \|_2 \leq \| \mathbf{q}_i \|_2 +\epsilon \\
    \Leftrightarrow &\; (\| \mathbf{q}_i \|_2 -\epsilon)^2 \leq \| \mathbf{p}_i + \mathbf{t} \|_2^2 \leq (\| \mathbf{q}_i \|_2 +\epsilon)^2. \label{eq:3DCorreq1}
\end{align}
Since $\|\mathbf{p}_i +\mathbf{t} \|_2^2$ can be separated into two additive parts, $h_1 (t_1,t_2) = (p_{1i}+t_1)^2+(p_{2i} + t_2)^2 $ and $h_2 (t_3) = (p_{3i} + t_3)^2$, this is a special case of Example~\ref{example: add}. In particular, according to Section~\ref{sec:CoreMethod}, the lower bound for ACM-1 here arises as 
\begin{equation}
\begin{aligned}
    L(\Cube_{ACM}) = &\; \max_{t_3} \sum_i \mathbf{1} ( |h_i(t_1^c, t_2^c; t_3)| \leq \epsilon) \\
    = &\; \max_{t_3} \sum_i \mathbf{1} (t_3 \in [\underline{t}_{3i}^l, \underline{t}_{3i}^r]),
\end{aligned}
\end{equation}
and the upper bound of ACM-1, can be derived similarly follows from Example~\ref{example: add} (see also the appendix for details):
\begin{equation}
\begin{aligned}
    U(\Cube_{ACM}) = &\; \max_{t_3} \sum_i \mathbf{1} ( |h_i(\Cube_{ACM}; t_3)| \leq \epsilon) \\
    = &\; \max_{t_3} \sum_i \mathbf{1} (t_3 \in [\overline{t}_{3i}^l, \overline{t}_{3i}^r]).
\end{aligned}
\end{equation}
\subsection{Correspondence-less Registration}\label{Sec:3DCorrlessProblem}

Limited by the quality and distinctiveness of point cloud feature descriptors, matching correspondences can be problematic. This has motivated researchers to design correspondence-less methods. For example, registration with all-to-all matching lists is considered in \cite{yang2020teaser}. With the observation that inner distances of a point set are invariant with respect to rigid transformations (cf. \figurename~\ref{fig:bunny1}), Liu et al.~\cite{liu2018efficient} propose to filter out most of the useless correspondences and then feed them to a plain BnB called GoTS to search the translation and correspondences simultaneously.

\begin{figure}
    \centering
    \includegraphics[width = 0.7\linewidth]{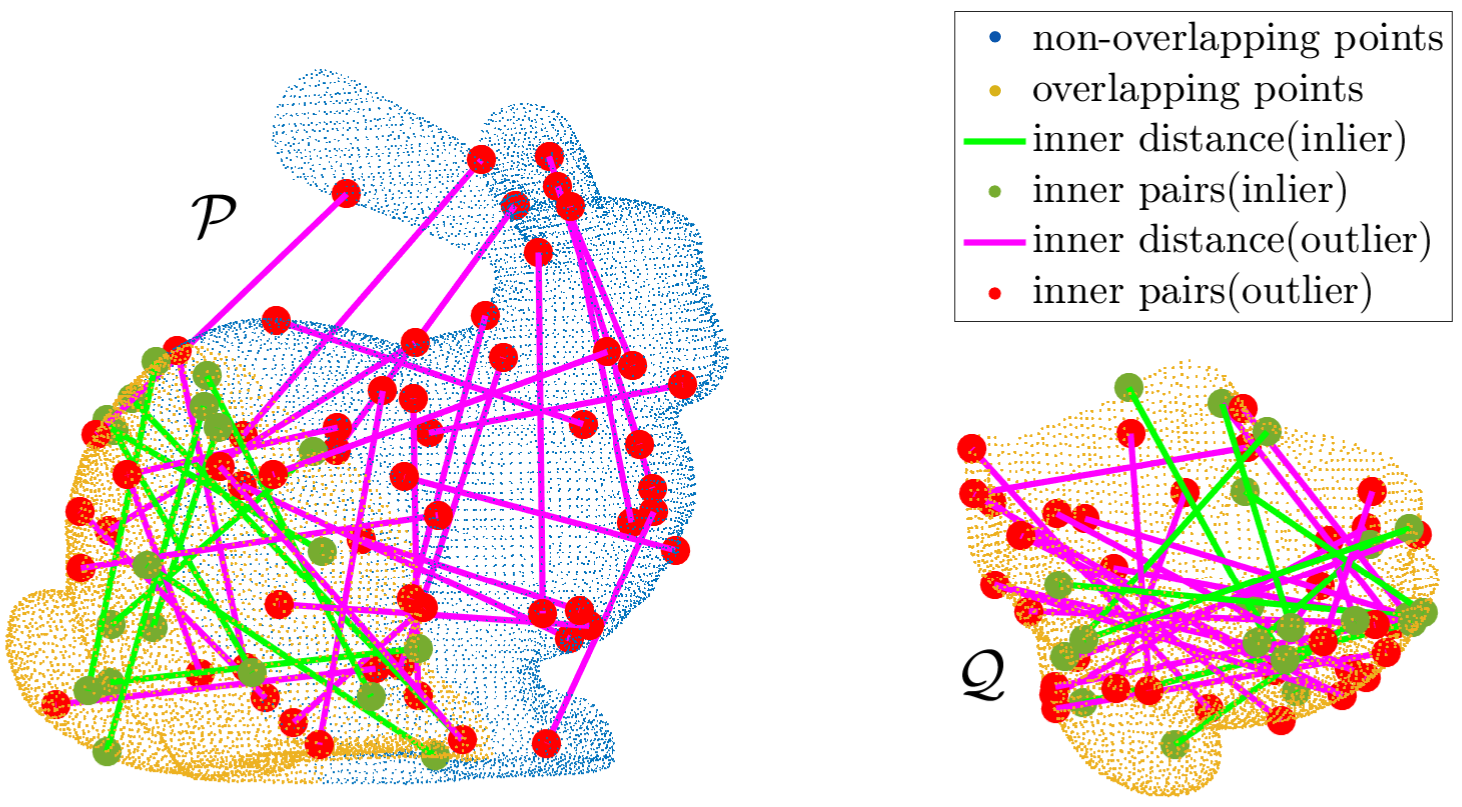}
    \caption{Visualization of the correspondence-less registration setup (Section~\ref{Sec:3DCorrlessProblem}) and the Bunny experiment (Section~\ref{subsection:3DProblem-real-experiments}). The full Bunny point set on the left can be regarded as the reference set. The partial Bunny point cloud on the right is the moving set and constitutes the overlapping part after a motion compensation. The length of the inter-pair line segments is invariant when point sets rigidly move.}
    \label{fig:bunny1}
\end{figure}

Formally, using the general problem formulation in (\ref{eq:CorrlessRigid}), suppose the point pairs ($\mathbf{q}_{i1^*},\mathbf{q}_{i2^*}$) are randomly picked from the overlapping point part of set $\mathcal{Q}$ and the pairs ($\mathbf{p}_{j1^*},\mathbf{p}_{j2^*}$) are the corresponding ground truth pairs from the reference point set $\mathcal{P}$. One key observation is that 
\begin{equation}
\begin{aligned}    
    \| \mathbf{q}_{i1^*} - \mathbf{q}_{i2^*} \|_2 &\;= \| \mathbf{R}(\mathbf{p}_{j1^*} + \mathbf{t}) - \mathbf{R}(\mathbf{p}_{j2^*} + \mathbf{t}) \|_2 \\
    &\;= \| \mathbf{p}_{j1^*} - \mathbf{p}_{j2^*} \|_2, \label{eq:pairwise distance}
\end{aligned}
\end{equation}
where $\mathbf{R}\in SO(3)$ is any rotation and $\mathbf{t} \in \mathbb{R}^{3}$ any translation. 

For ground truth pairs ($\mathbf{q}_{i1^*},\mathbf{q}_{i2^*}$) and ($\mathbf{p}_{j1^*},\mathbf{p}_{j2^*}$), the distance $| \| \mathbf{q}_{i1^*} - \mathbf{q}_{i2^*} \|_2 - \| \mathbf{p}_{j1^*} - \mathbf{p}_{j2^*} \|_2 |$ is expected to be small, so, with an appropriate  threshold $\tau$, we can then keep only the pair $\big( (\mathbf{q}_{i1},\mathbf{q}_{i2}), (\mathbf{p}_{j1},\mathbf{p}_{j2}) \big)$ that satisfies
\begin{equation}
    |\| \mathbf{q}_{i1} - \mathbf{q}_{i2} \|_2 - \| \mathbf{p}_{j1} - \mathbf{p}_{j2} \|_2 | \leq \tau,
\end{equation}
for every $i1,i2 \in \{1,...,M_1\}$ and every $j1, j2\in \{1,...,M_2\}$ ($i1\neq i2, j1\neq j2$), and remove all other pairs. The remaining point pairs $\big( (\mathbf{q}_{i1},\mathbf{q}_{i2}), (\mathbf{p}_{j1},\mathbf{p}_{j2}) \big)$ are called Rotation Invariant (RI) pairs.

Next, observe that $\|\mathbf{q}_{i}\|=\|\mathbf{p}_{j} + \mathbf{t}^* \|$ further provides constraints for RI pairs. Using the $\ell_\infty$ norm, we have
\begin{equation}
    \| \mathbf{F}_{m,n} (\mathbf{t}^*) \|_\infty \leq \epsilon,
\end{equation}
\begin{equation}
    \mathbf{F}_{m,n}(\mathbf{t}):= 
    \begin{bmatrix} 
    \| \mathbf{p}_{i1}+\mathbf{t} \|_2 - \|\mathbf{q}_{j1}\|_2 \\
    \|\mathbf{p}_{i2}+\mathbf{t}\|_2 -\|\mathbf{q}_{j2}\|_2
    \end{bmatrix},
\end{equation}
where $m:=(i1,i2),n:=(j1,j2)$.
We can then formulate the Consensus Maximization problem as:
\begin{equation}
\begin{split}
     \max_{\mathbf{t}, \mathcal{I}} \;&\;|\mathcal{I}| \\
    s.t. &\; \| \mathbf{F}_{m,n} (\mathbf{t}) \|_\infty  \leq  \epsilon, \\
    &\;\forall \; \big ( (\mathbf{p}_{j1},\mathbf{p}_{j2}),(\mathbf{q}_{i1},\mathbf{q}_{i2}) \big ) \in \mathcal{I}.
\end{split}
\label{Prob:3DCorrless}
\end{equation}
Note that the constraint $\| \mathbf{F}_{m,n} (\mathbf{t}) \|_\infty \leq \epsilon$---
which aims to check whether the largest element between $|\|\mathbf{p}_{i1}+\mathbf{t} \|_2 - \|\mathbf{q}_{j1}\|_2|$ and
$|\|\mathbf{p}_{i2}+\mathbf{t}\|_2 -\|\mathbf{q}_{j2}\|_2|$ is smaller than $\epsilon$---
can also be interpreted as a check whether $|\|\mathbf{p}_{i1}+\mathbf{t} \|_2 - \|\mathbf{q}_{j1}\|_2|\leq \epsilon$ and
$|\|\mathbf{p}_{i2}+\mathbf{t}\|_2 -\|\mathbf{q}_{j2}\|_2|\leq \epsilon$ at the same time.

\begin{remark}\label{Remark: Corrless relax}
    The problem formulated in (\ref{Prob:3DCorrless}) is not exactly the same as the one in \cite{liu2018efficient},
     where they solve 
    \begin{equation}
        \max_{\mathbf{t}} \, \sum_{m} \max_{n} \mathbf{1}(\| \mathbf{F}_{m,n} (\mathbf{t}) \|_\infty  \leq  \epsilon).
    \label{Prob:3DCorrless Original}
    \end{equation}
    Note that the main difference between (\ref{Prob:3DCorrless}) and (\ref{Prob:3DCorrless Original}) is the number of counted inliers.
    As can be observed, (\ref{Prob:3DCorrless Original}) seeks a unique correspondence for each $m$.
    It is easy to prove though that the solution of (\ref{Prob:3DCorrless}) is also the solution of (\ref{Prob:3DCorrless Original})
    (see the appendix).
\end{remark} 

\noindent\textbf{Plain BnB.}
 Given the cube $\Cube_{plain}$ with centre point $\mathbf{t}_c$ and half of its diameter $r_{plain}$, the lower bound is again trivially given by using $\mathbf{t}_c$. To provide a clear explanation on the upper bound, let us use the new index notation $u:=(i1,j1), v:=(i2,j2)$. For any translation $\mathbf{t}$ of the cube $\Cube_{plain}$ and $\mathbf{x}\in \mathbb{R}^3$, we have a basic geometric inequality
 \begin{equation}
    \| \mathbf{x}+ \mathbf{t}_c \|_2 - r_{plain} \leq \|\mathbf{x} + \mathbf{t} \|_2 \leq \| \mathbf{x} + \mathbf{t}_c \|_2 + r_{plain}.\label{eq:Geometric3D}
\end{equation}

Therefore, for each pair $(\mathbf{p}_{i1}, \mathbf{q}_{j1})$, writing $c_{u} := \|\mathbf{q}_{j1}\|_2 - \|\mathbf{p}_{i1}+\mathbf{t}_c \|_2$, we can get 
\begin{align}
    c_u - r_{plain} \leq \|\mathbf{q}_{j1}\|_2 &\;- \|\mathbf{p}_{i1}+\mathbf{t} \|_2 \leq c_u + r_{plain},\\
    c_u^l \leq |\|\mathbf{q}_{j1}\|_2 &\;- \|\mathbf{p}_{i1}+\mathbf{t} \|_2| \leq c_u^r,
\end{align}
where $c_u^l = \min\{ 0,  |c_u - r_{plain}|, |c_u + r_{plain}|\}$ and $c_u^r = \max\{ |c_u - r_{plain}|, |c_u + r_{plain}|\}$. Similarly, denote $c_v = \|\mathbf{q}_{j2}\|_2 - \|\mathbf{p}_{i2}+\mathbf{t}_c \|_2$,
$c_v^l = \min\{ 0,  |c_v - r_{plain}|, |c_v + r_{plain}|\}$ and $c_v^r = \max\{ |c_v - r_{plain}|, |c_v + r_{plain}|\}$. We now can obtain 
\begin{equation}
\begin{aligned}
    \| \mathbf{F}_{m,n}  (\Cube_{plain})\|_\infty &\;= \max \{ [c_u^l,c_u^r],[c_v^l,c_v^r] \}\\
    &\; = [\max\{ c_u^l,c_v^l \}, \max\{ c_u^r,c_v^r \}].
\end{aligned}
\end{equation}
Therefore, the upper bound of plain BnB is defined by
\begin{equation}
    \begin{aligned}
        U(\Cube_{plain}) &\; = \sum_{u,v} \textbf{1} ([\max\{ c_u^l,c_v^l \}, \max\{ c_u^r,c_v^r \}] \cap [0,\epsilon])\\ 
        &\; = \sum_{u,v} \textbf{1} (\max\{c_u^l,c_v^l\} \leq \epsilon)\\
        &\; \geq \max_{\mathbf{t}} \sum_{u,v} \textbf{1} (\|\mathbf{F}_{u,v}  (\mathbf{t})\|_\infty.
    \end{aligned}
\end{equation}
\noindent\textbf{ACM-2 Corrless.}
Without the loss of generality, suppose ACM-2 Corrless branches over the first two dimensions, and uses Interval Stabbing to find $t_3$.
Given the cube $\Cube_{ACM}$ with centre point $\mathbf{t}_{c(1:2)}$ and its half-diameter $r_{ACM}$,
consider the following sub-constraint for each pair $(\mathbf{p}_{i1}, \mathbf{q}_{j1})$
\begin{align}
    &\;| \| \mathbf{p}_{i1} + \mathbf{t} \|_2 - \| \mathbf{q}_{j1} \|_2 | \leq \epsilon \\
    \Leftrightarrow &\; \| \mathbf{p}_{i1} + \mathbf{t} \|_2^2 \cap [(-\epsilon+ \|\mathbf{q}_{j1} \|_2)^2, (\epsilon + \|\mathbf{q}_{j1} \|_2)^2]. \label{eq:sub-contraint}
    \end{align}
In order to separate out $t_3$, let us define 
\begin{align} 
h_1(\mathbf{p}_{i1(1:2)},\mathbf{t}_{(1:2)}) &\;= \|\mathbf{p}_{i1(1:2)}+\mathbf{t}_{(1:2)}\|_2^2 \\
h_2(\mathbf{p}_{i1(3)},t_3) &\;= \|\mathbf{p}_{i1(3)}+t_3\|_2^2,
\end{align}
such that $\| \mathbf{p}_{i1} + \mathbf{t} \|_2^2 = h_1(\mathbf{p}_{i1(1:2)},\mathbf{t}_{(1:2)}) + h_2(\mathbf{p}_{i1(3)},t_3)$.

Again, the lower bound of ACM-2 Corrless is modified from its pendant in plain BnB.
Given $\mathbf{t}_{c(1:2)}$ and substituting $h_1(\mathbf{p}_{i1(1:2)},\mathbf{t}_{c(1:2)})$ into (\ref{eq:sub-contraint}), an interval $[\underline{t}_{3u}^l,\underline{t}_{3u}^r]$ for $t_3$ is again easily found by function inversion.
Similarly, on the same cube we can derive another interval $ [\underline{t}_{3v}^l,\underline{t}_{3v}^r]$ for $t_3$ from the sub-constraint on $(\mathbf{p}_{i2}, \mathbf{q}_{j2})$. 
Therefore, the trivial lower bound of ACM-2 Corrless is given by
\begin{equation}
    L(\Cube_{ACM})= \max_{t_3} \;\sum_{u,v} \textbf{1} (t_3\in [\underline{t}_{3u}^l,\underline{t}_{3u}^r] \cup [\underline{t}_{3v}^l,\underline{t}_{3v}^r] ),
\end{equation}
which can be solved globally by Interval Stabbing.

The upper bound of ACM-2 Corrless is derived in the same manner as in Example~\ref{example: add}.
Note that we can also get the $2$-dimensional version of (\ref{eq:Geometric3D}), denote the interval of $h_1(\mathbf{p}_{i1(1:2)},\mathbf{t}_{(1:2)})$ on $\Cube_{ACM}$ as
\begin{equation}
    h_{1u}^l \leq h_1(\Cube_{ACM},\mathbf{t}_{(1:2)}) \leq h_{1u}^r.
\end{equation}
Relaxation on sub-constraint (\ref{eq:sub-contraint}) is achieved by counting the following intervals
\begin{equation}
    \begin{aligned}
        &\; h_2(\mathbf{p}_{i1(3)},t_3) \\ 
        \in &\; [(-\epsilon+ \|\mathbf{q}_{j1} \|_2)^2 - h_{1u}^r, (\epsilon + \|\mathbf{q}_{j1} \|_2)^2- h_{1u}^l] \\
        \Rightarrow &\; t_3 \in [\overline{t}_{3u}^l,\overline{t}_{3u}^r].
    \end{aligned}
\end{equation}
Similarly, we can also derive the relaxed interval of $t_3$ on $(\mathbf{p}_{i2}, \mathbf{q}_{j2})$ and denote it as $[\overline{t}_{3v}^l,\overline{t}_{3v}^r]$. 
Finally, a valid upper bound of ACM-2 Corrless can be found by doing Interval Stabbing on the problem
\begin{equation}
    U(\Cube_{ACM}) = \max_{t_3} \; \mathbf{1} (t_3 \in [\overline{t}_{3u}^l,\overline{t}_{3u}^r] \cup [\overline{t}_{3v}^l,\overline{t}_{3v}^r]).
\end{equation}
Once the optimal translation $\mathbf{t}^*$ is found, the unknown rotation $\mathbf{R}^*$ can be efficiently estimated by existing algorithms. For example,~\cite{parra2014fast} proposed a $3$-dimensional BnB with fast search, and ~\cite{peng2022arcs} proposed an efficient and accurate method to simultaneously search rotation and correspondences.

\begin{figure*}
    \centering
    \subfloat[Running Time\label{fig:3DCorrtime}]{%
       \includegraphics[height=.2\linewidth]{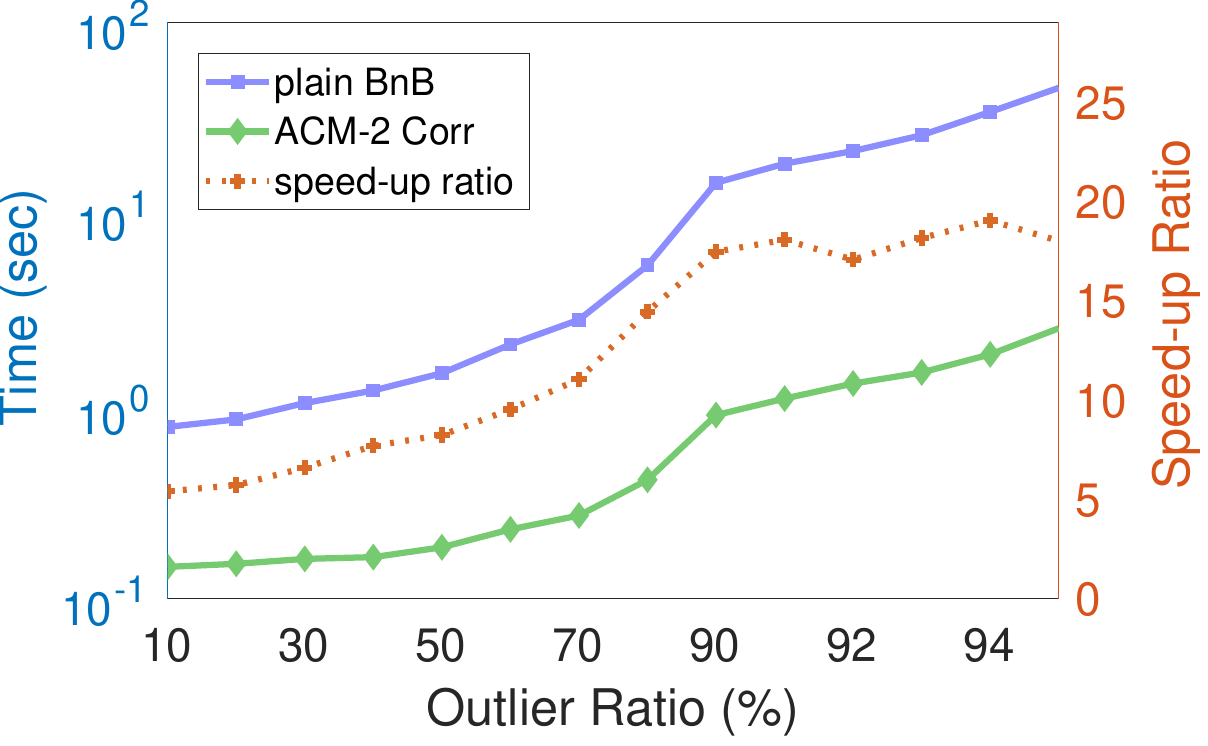}}
       \hfill
    \subfloat[Translation Error\label{fig:3DCorrerror}]{%
       \includegraphics[height=.2\linewidth]{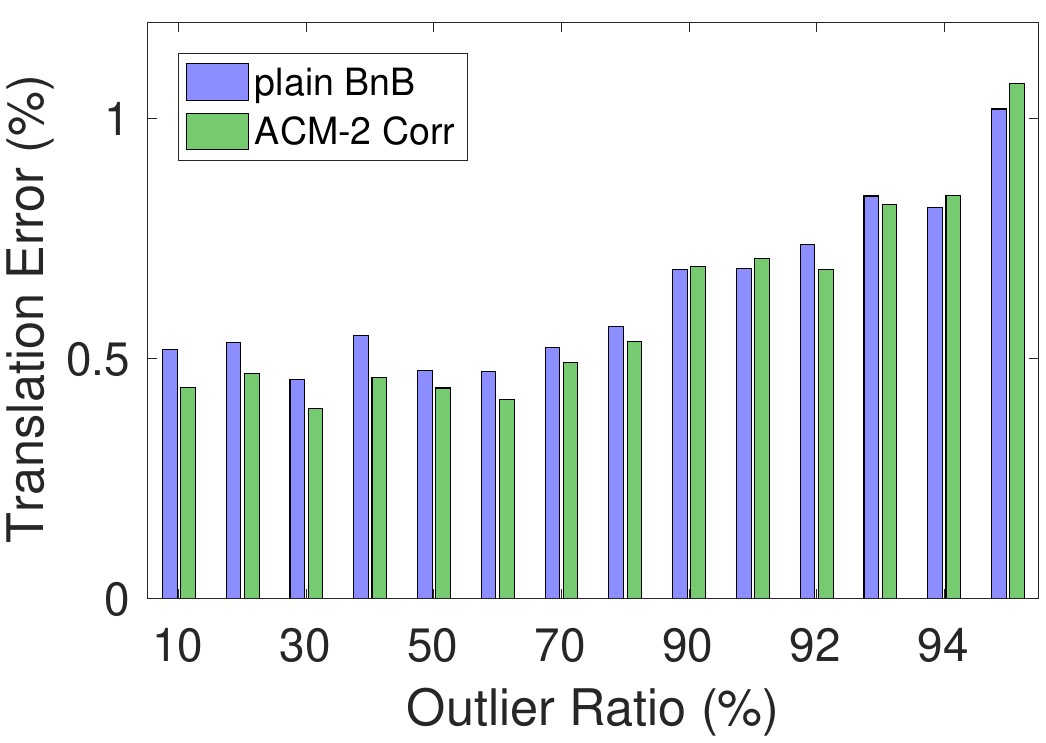}}
       \hfill
    \subfloat[Translation Error\label{fig:3DCorriter}]{%
       \includegraphics[height=.2\linewidth]{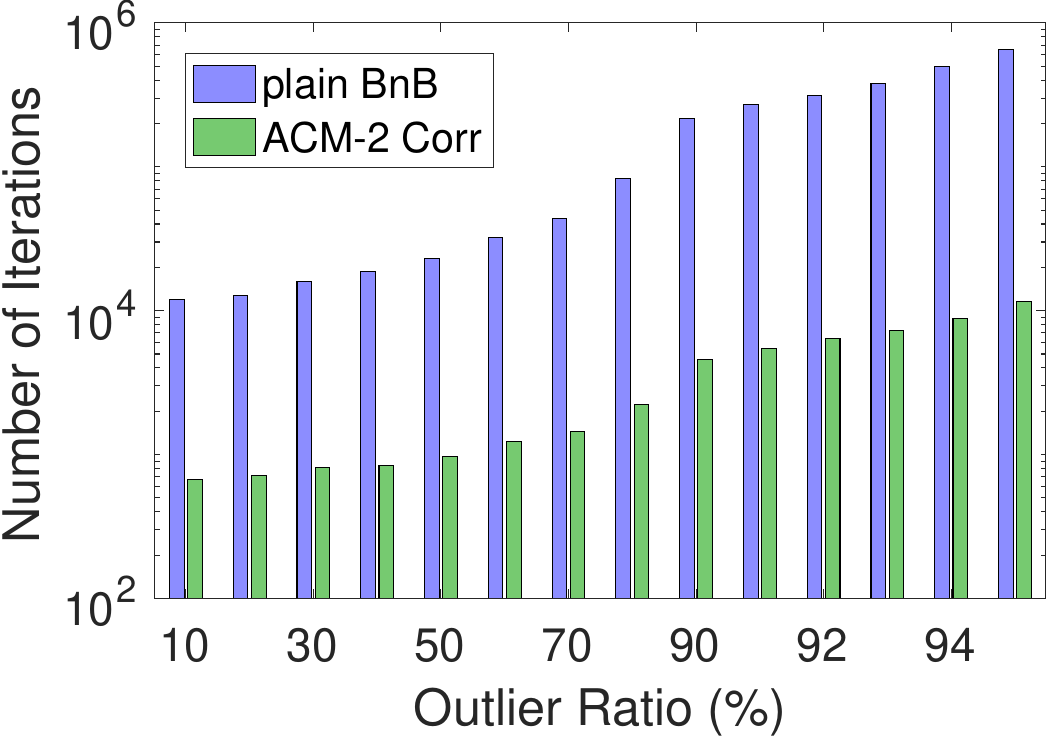}}
       
    \subfloat[Running Time\label{fig:3DCorrlesstime}]{%
       \includegraphics[height=.2\linewidth]{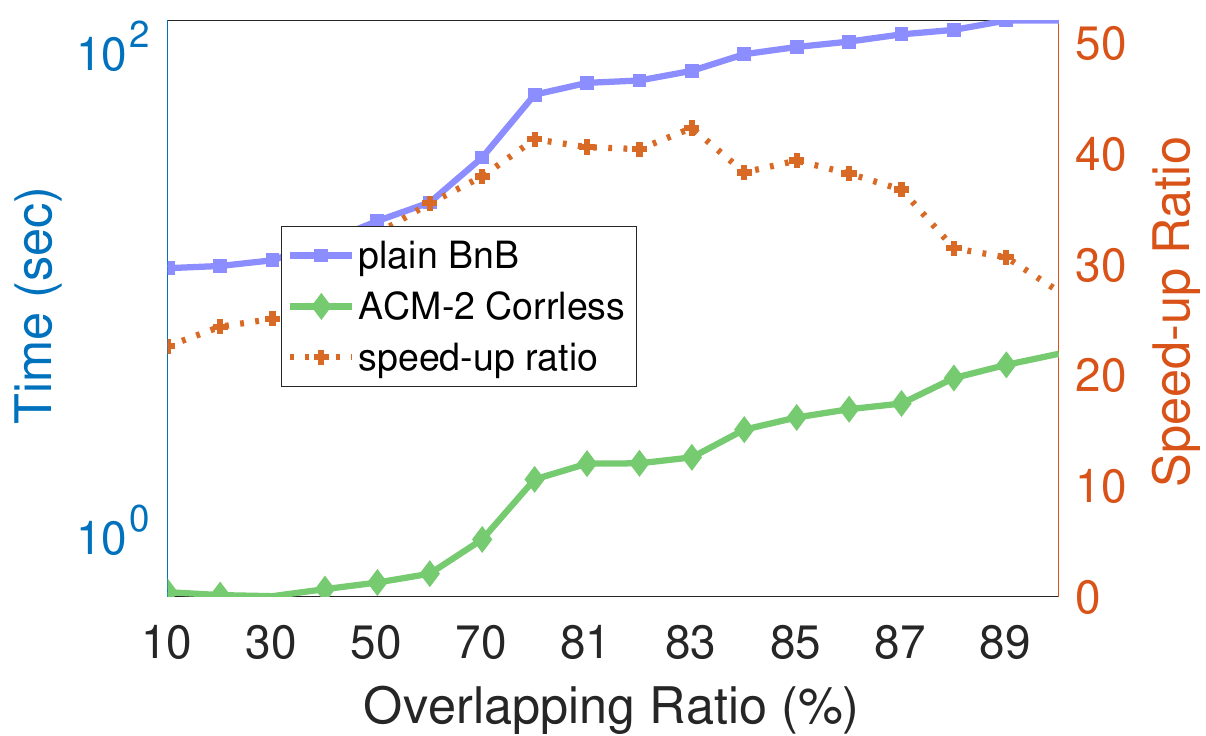}}
       \hfill
    \subfloat[Translation Error\label{fig:3DCorrlesserror}]{%
       \includegraphics[height=.2\linewidth]{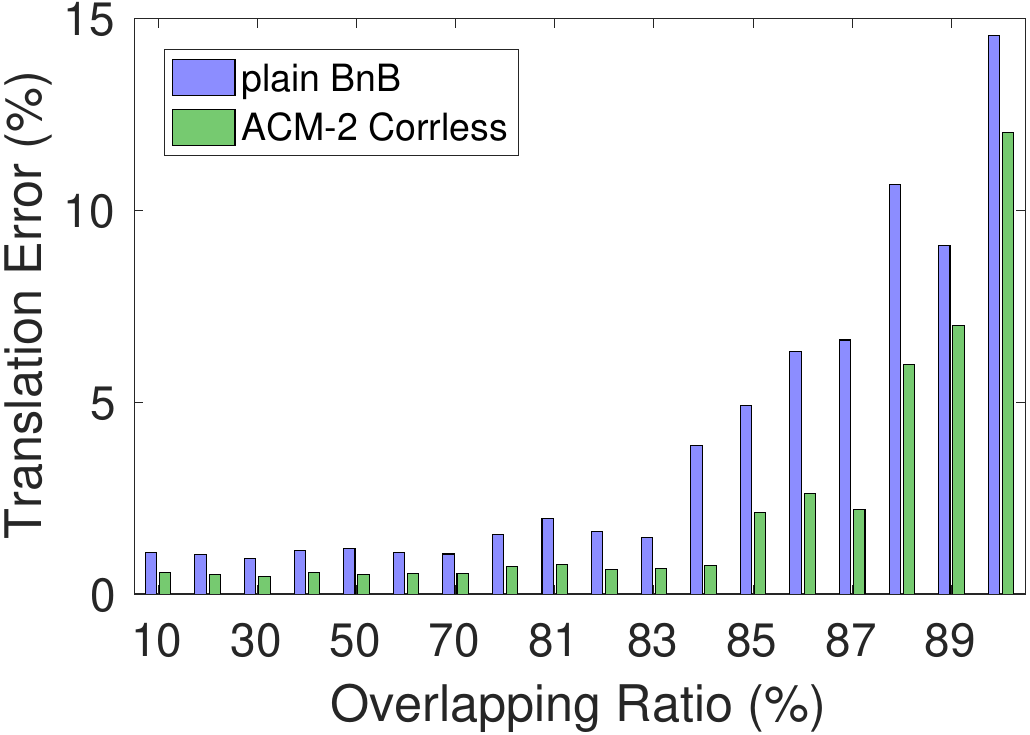}}
       \hfill
    \subfloat[Translation Error\label{fig:3DCorrlessiter}]{%
       \includegraphics[height=.2\linewidth]{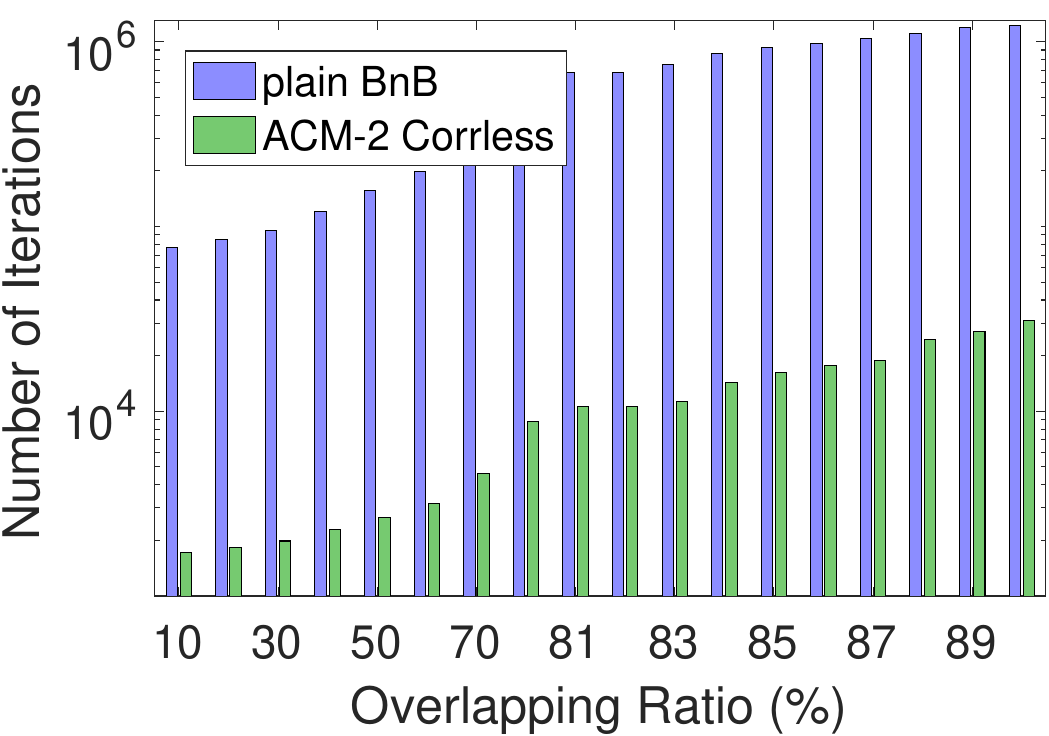}}
    % \caption{3D-3D registration}
    \caption{ 3D-3D registration experiments (Section~\ref{subsection:3DProblem-synthetic-experiments}). The first row shows the correspondence-less setting (Section~\ref{Sec:3DCorrProblem}) and the second row shows the correspondence-based setting (Section~\ref{Sec:3DCorrlessProblem}). ACM-2 Corr is $20 \times$ faster than plain BnB in the high-outlier regime (\figurename~\ref{fig:3DCorrtime}) with comparable translation errors (\figurename~\ref{fig:3DCorrerror}) and converging within much fewer iterations (\figurename~\ref{fig:3DCorriter}). Similar results are shown in \figurename~\ref{fig:3DCorrlesstime}-\figurename~\ref{fig:3DCorrlessiter} for the correspondence-less setting, where ACM-2 Correless attains  $30 \times$ - $40 \times$ speed-up. ($1000$ randomly generated correspondences, $100$ trials) }
    \label{fig:3Dsimulation}
\end{figure*}

\subsection{Synthetic Experiments} \label{subsection:3DProblem-synthetic-experiments}
\noindent\textbf{Setup.} 
The synthetic data generation for 3D-3D registration is similar to that in Section~\ref{sec:3d2dsynth}, except that we do not transform them into 2D features here. 
To introduce outliers, we randomly generate 3D points within the same range, transform them into one of the coordinate systems, and use them to replace the original inlier points.
The relative translation and rotation axis are randomly generated in the unit cube. The rotation angle is randomly sampled in the interval $[-\pi, \pi]$.
For correspondence-based registration, we first generate 3D-3D correspondences and then add noise from a Gaussian distribution with zero mean and 0.01 variance to the 3D points in both coordinate systems. 
For correspondence-less registration, we generate two noiseless point sets of the same size.

In both settings, we set the threshold $\epsilon = 0.001$ and matching threshold for correspondence-based setting $\tau = 0.1\epsilon$, the number of correspondences/points in each set to $1000$, and we run $100$ trials to get stable results. The maximal depth of plain BnB and ACM-2 Corr is set to $10$. 
As suggested in \cite{liu2018efficient}, we keep $1000$ point pairs with the largest inter-point distance in each point set for practical implementation in the correspondence-less setting.
To measure translation errors, we define the relative translation error as $\frac{\|\hat{\mathbf{t}} -\mathbf{t}_{gt}\|_2}{\| \mathbf{t}_{gt} \|_2}$, where $\mathbf{t}_{gt}$ is the ground truth translation and $\hat{\mathbf{t}}$ is the estimated translation.

\noindent\textbf{Results.}
For the correspondence-based setting, the outlier ratio is varied from $10\%$ to $90\%$ in steps of $10\%$, and then from $91\%$ to $95\%$ in steps of $1\%$. The mean results are displayed in \figurename~\ref{fig:3Dsimulation}. 
For the correspondence-less setting, the overlapping ratio is varied from $10\%$ to $70\%$ in steps of $10\%$, and then from $81\%$ to $90\%$ in steps of $1\%$. 

\figurename~\ref{fig:3DCorrtime} and~\ref{fig:3DCorrlesstime} make it clear that significant improvements have been obtained in both settings.
As shown, ACM-2 yields an order of magnitude faster results than plain BnB and the number of iterations of ACM-2 is about $2$ orders of magnitude smaller than that of plain BnB. Similar errors are observed in all cases.
When dealing with more extreme outlier ratios (larger than $90\%$) or extreme overlapping ratios (larger than $80\%$), ACM-2 Corr and ACM-2 Corrless are around $18 \times$ and $40\times$ faster than plain BnB, respectively. These observations are in line with the previously discussed advantage of ACM that the less $1$-dimensional search space helps ACM to converge much faster.

\subsection{Real Experiments} \label{subsection:3DProblem-real-experiments}
Here we use the Stanford \textit{bunny} point cloud \cite{curless1996volumetric} to compare ACM-2 Corrless and plain BnB.

\noindent\textbf{Setup.} Following ~\cite{liu2018efficient}, we use the \textit{pcdownsample}~\cite{pomerleau2013comparing} function from MATLAB to down-sample the original point cloud. To generate point sets, we then randomly cut the original bunny according to a certain overlap ratio and thereby obtain a fragment of the original point set. To conclude, we randomly transform this small fragment. The translation is randomly generated in $[-1,1]^3$ and the rotation angles are randomly sampled from $[-\pi, \pi]$. 
\figurename~\ref{fig:bunny1} shows an example experiment setup with inliers and outliers as identified through the CM.

To get RI pairs, we follow the same strategy as in~\cite{liu2018efficient}.
We set the threshold to $10^{-4}$ and the maximal depth to $10$. The initial cube of plain BnB is set to $[-1,1]^3$ and $[-1,1]^2$ for ACM-2 Corrless.

\noindent\textbf{Results.}
\figurename~\ref{fig:BunnyResults} shows a comparison between plain BnB and ACM-2 Corrless, where we vary the overlap ratio from $10\%$ to $90\%$ (cf. first row), and then from $1\%$ to $9\%$ (cf. second row) to check extreme cases.
ACM-2 Corrless presents a significant overall speed-up with respect to plain BnB while the translation errors remain comparable. To be specific, ACM-2 Corrless converges in less than $0.5$ sec on average, whereas plain BnB needs more than $10$ sec. The speed-up factor for ACM-2 Corrless therefore is about $40 \times$.
Interestingly, both methods can provide accurate estimation even in extreme cases where ACM-2 Corrless is more stable than plain BnB. 
As can be identified in \figurename~\ref{fig:Bunnyerror1}, both methods can work in the majority of cases down to overlap ratios as low as $4\%$.

\begin{figure}[h]
    \centering
    \subfloat[\small Small Ratio Running Time\label{fig:Bunnytime1}]{%
       \includegraphics[height=.36\linewidth]{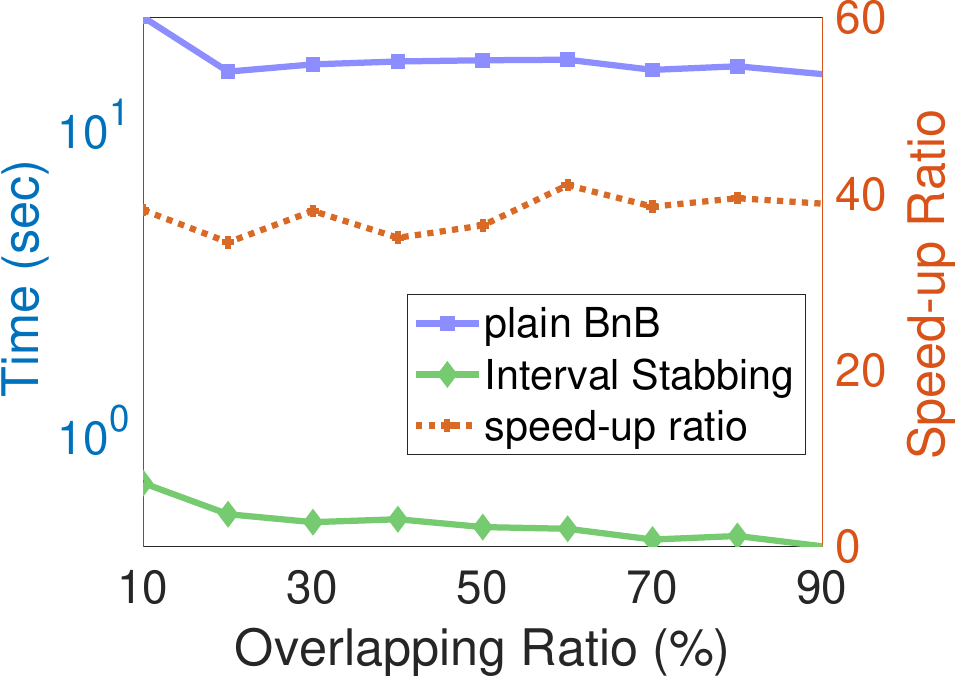}}
       \hfill
    \subfloat[\small Small Ratio Error\label{fig:Bunnyerror1}]{%
       \includegraphics[height=.36\linewidth]{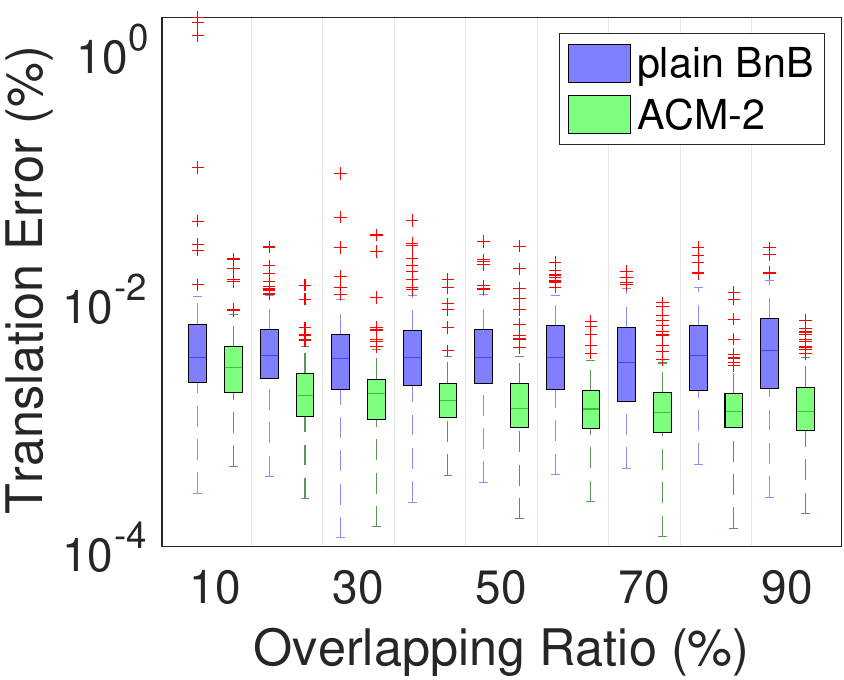}}
       
    \subfloat[{\small Large Ratio Running Time}\label{fig:Bunnytime2}]{%CoverLetter-MinorRevision.tex CoverLetter.tex
       \includegraphics[height=.36\linewidth]{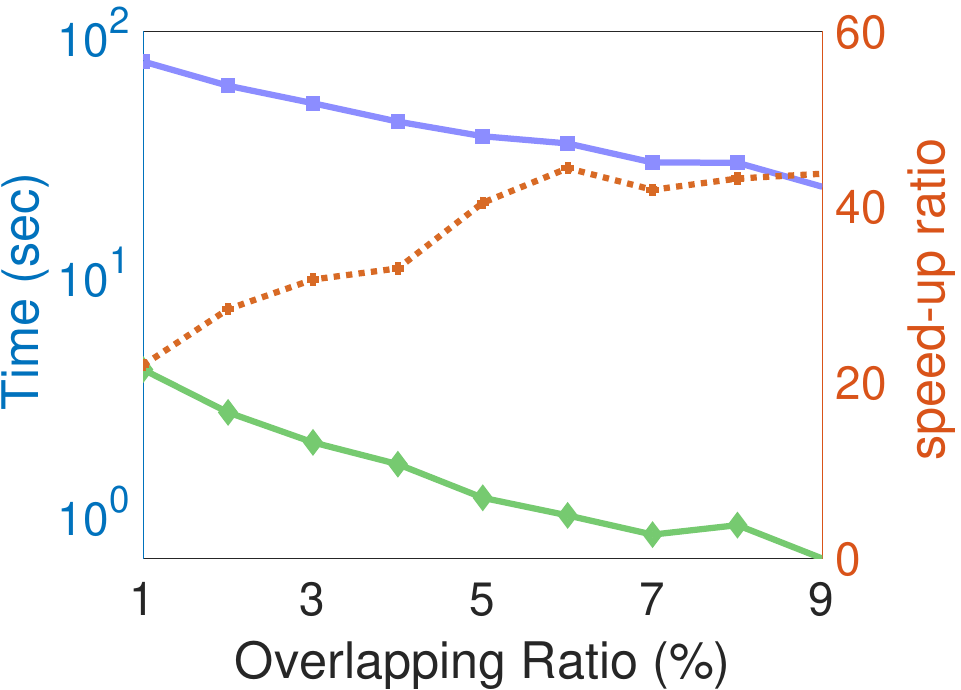}}
       \hfill
    \subfloat[{\small Large Ratio Error}\label{fig:Bunnyerror2}]{%
       \includegraphics[height=.36\linewidth]{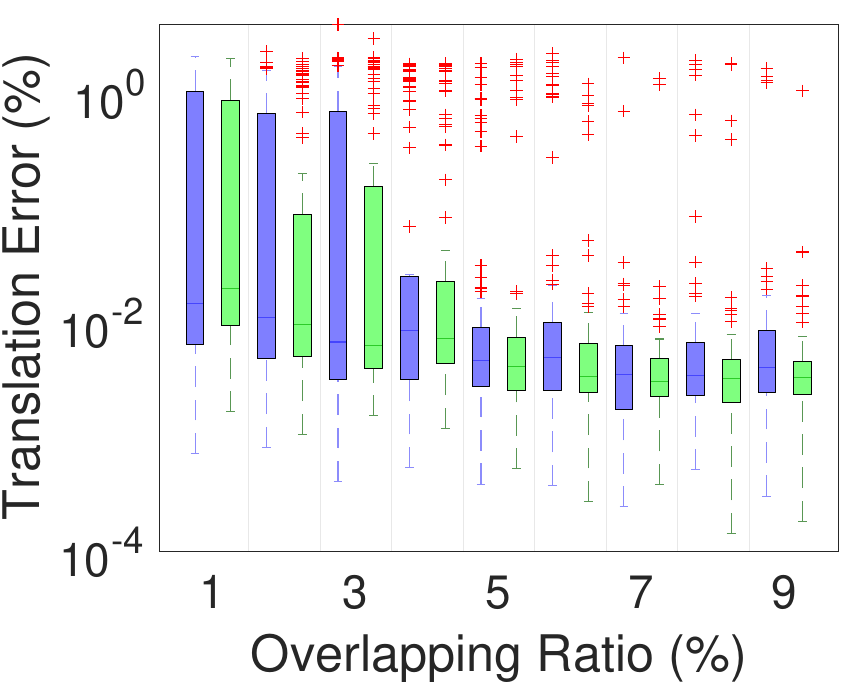}}

   \subfloat[Number of Iterations\label{fig:Bunnyiter}]{%
       \includegraphics[height=.37\linewidth]{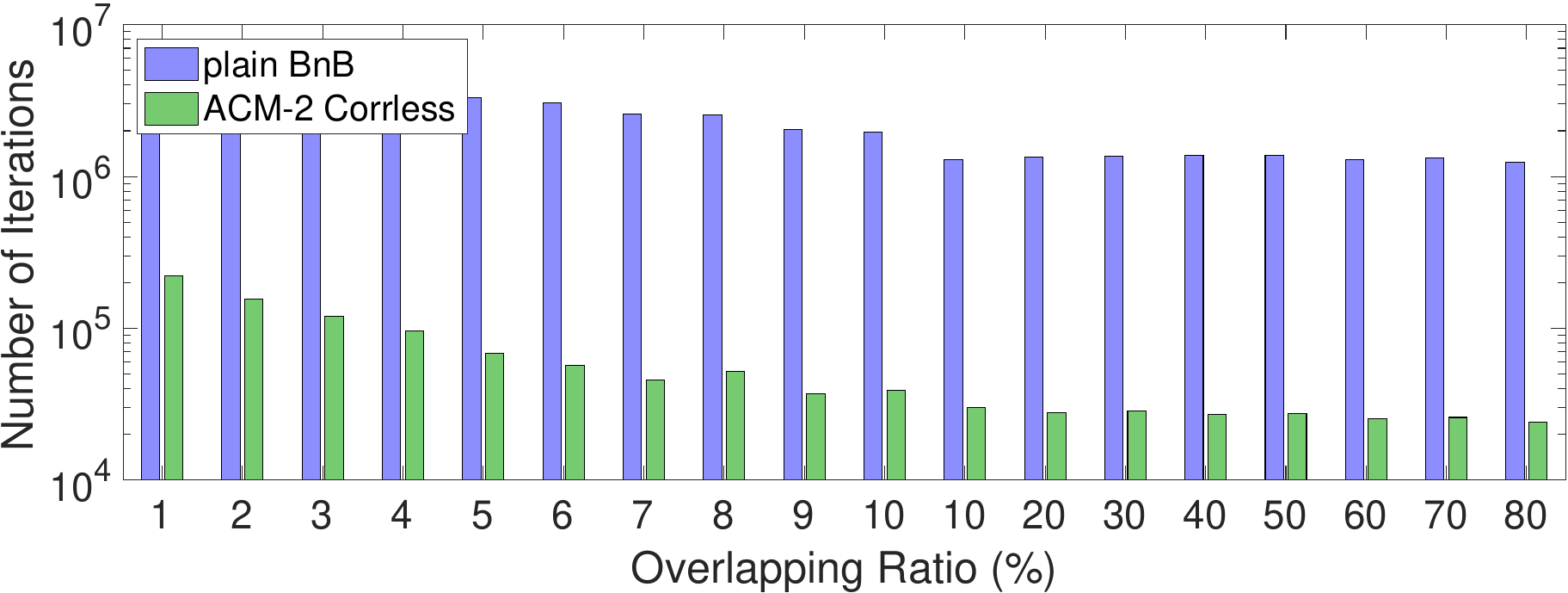}}
    \caption{  3D-3D correspondence-less registration experiments on Stanford Bunny (Section~\ref{subsection:3DProblem-real-experiments}). ACM-2 Corr is $30 \times$ - $40 \times$ faster than plain BnB (\figurename~\ref{fig:Bunnytime1}~\ref{fig:Bunnytime2}) with comparable translation errors (\figurename~\ref{fig:Bunnyerror1}~\ref{fig:Bunnyerror2}) and converging within much fewer iterations (\figurename~\ref{fig:Bunnyiter}). (Mean results on $100$ trials) }
    \label{fig:BunnyResults}
\end{figure}
\begin{remark}
    Note that the main contribution of this work is to an acceleration of consensus maximization, which we test on the noiseless, correspondence-less scenario addressed in \cite{liu2018efficient}. Though this already serves well to highlight the advantages of our method, we nonetheless want to refer the reader to Zhang et al.~\cite{zhang2021self} who introduce an extension of the approach to handle noise with the help of an initial state estimation.
\end{remark}

\section{ACM-3: Rotation and Focal Length Estimation} \label{sec:4D}
In this section, we provide a $4$-dimensional application to further demonstrate the performance of ACM.
Consider a camera rotating about its optical centre (e.g., panoramic stitching); furthermore, the focal length is unique but unknown in both views. A $4$D problem then naturally arises as we need to find a 3D rotation and an unknown focal length parameter. Such $4$D problem was studied in \cite{brown2007minimal} and can be solved by RANSAC. In \cite{bazin2014globally}, Bazin et al. proposed a plain BnB method to solve the $4$D problem globally optimally.

\noindent \textbf{Problem Formulation.} Denote by $( \mathbf{x}_i, \mathbf{x}'_i )$ a pair of centered $2D$ image points and by $(\tilde{\mathbf{x}}_i,\tilde{\mathbf{x}}'_i)$ the corresponding homogenized coordinates as in Sec.~\ref{Sec:2DProblem}. When a camera rotates around its optical centre, there is no translation involved, and the unknown variables consist of a 3D rotation $\mathbf{R}\in SO(3)$ and a calibration matrix $\mathbf{K}\in \mathbb{R}^{3\times 3}$. If $(\tilde{\mathbf{x}}_i,\tilde{\mathbf{x}}'_i)$ is a noiseless inlier pair, they are related by a homography at infinity as per

\begin{equation}
    \tilde{\mathbf{x}}' \sim \mathbf{K} \mathbf{R} \mathbf{K}^{-1} \tilde{\mathbf{x}}.
\end{equation}
Note that since the image points are centred,  $\mathbf{K}$ can be written in the form of $diag([f,f,1])$, where $f$ denotes the unknown focal length. As a result, the consensus maximization problem can be formulated as 
\begin{equation}
    \begin{aligned}
        \max_{\mathcal{I},\mathbf{R},f} &\; |\mathcal{I}| \\
        s.t. &\; \|\mathbf{x}'_i-\mathbf{\pi}(\mathbf{K}\mathbf{R}\mathbf{K}^{-1}\tilde{\mathbf{x}}_i)\|\leq \epsilon,
        \forall ( \mathbf{x}_i, \mathbf{x}'_i ) \in \mathcal{I}. \label{Problem: 4D}
    \end{aligned}
\end{equation}
Recall that, similarly to Sec.~\ref{Sec:1DProblem}, the rotation $\mathbf{R}$ can be parameterised by $3$ unknown angles as $\mathbf{R}_z(\theta) \mathbf{R}_y(\alpha) \mathbf{R}_z(\phi)$. Hence, one can search the $4$D space of $\theta,\alpha,\phi$ and $f$ and solve this consensus maximization problem by plain BnB globally optimally. We refer the reader to Bazin et al. \cite{bazin2014globally} for details of this BnB method, where an appropriate hyperbola was designed to develop upper bounds.

We build upon their work \cite{bazin2014globally} and develop ACM-3, an accelerated consensus maximization method that branches over the 3D space of $\alpha, \phi, f$ for finding an optimal solution. As usual, the remaining degree of freedom $\theta$ is determined using Interval Stabbing. While we need to again derive tailored upper and lower bounds for ACM-3, the derivation is relatively similar to that in Secs.~\ref{Sec:1DProblem} and~\ref{Sec:2DProblem}. Therefore, for the sake of brevity, we omit the derivation here.

\begin{figure*}
    \centering
    \subfloat[\scalebox{.90}{Running Time}\label{fig:4Dtime}]{%
       \includegraphics[height=.17\linewidth]{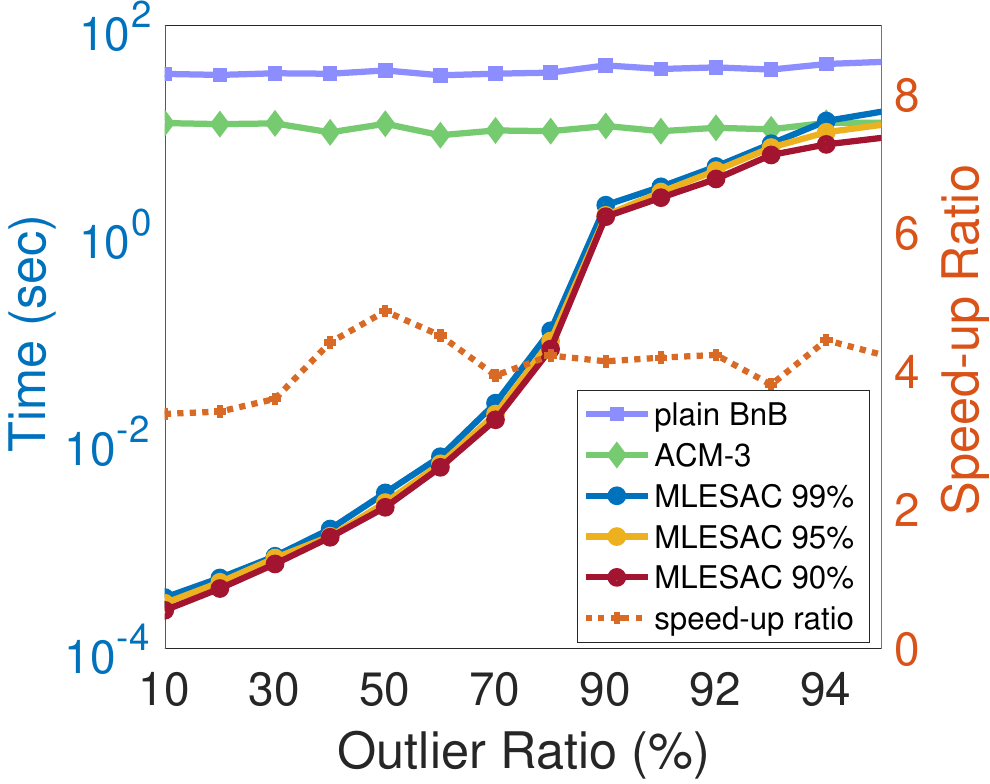}}
       % \hfill
    \subfloat[\scalebox{.90}{Rotation Error}\label{fig:4DerrorR}]{%
       \includegraphics[height=.17\linewidth]{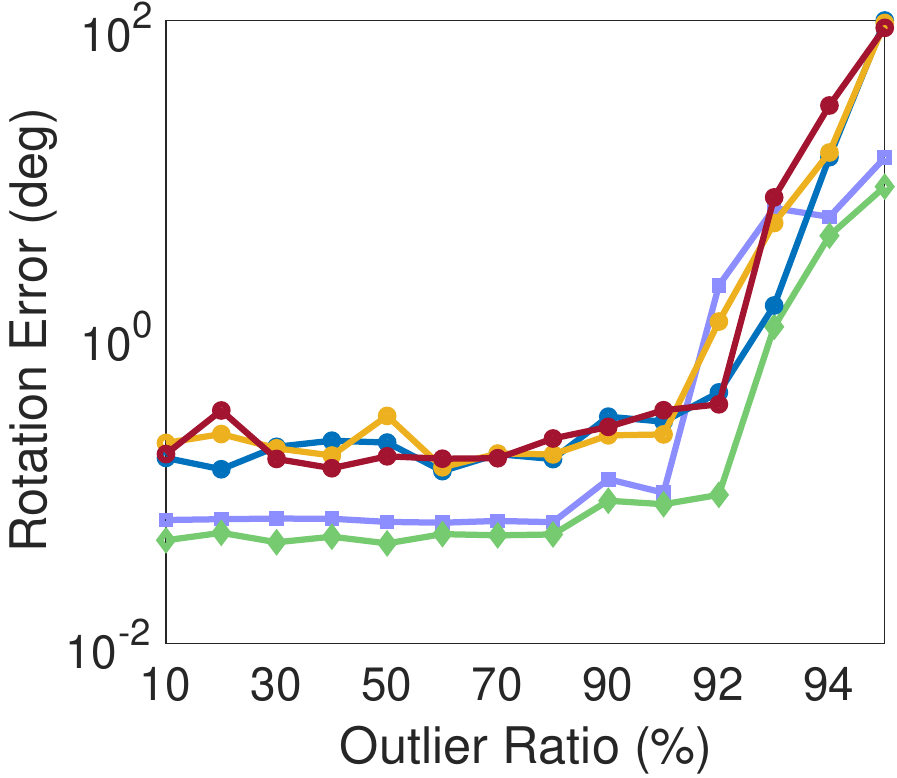}}
       % \hfill
    \subfloat[\scalebox{.90}{Focal Length Error}\label{fig:4Derrorf}]{%
       \includegraphics[height=.17\linewidth]{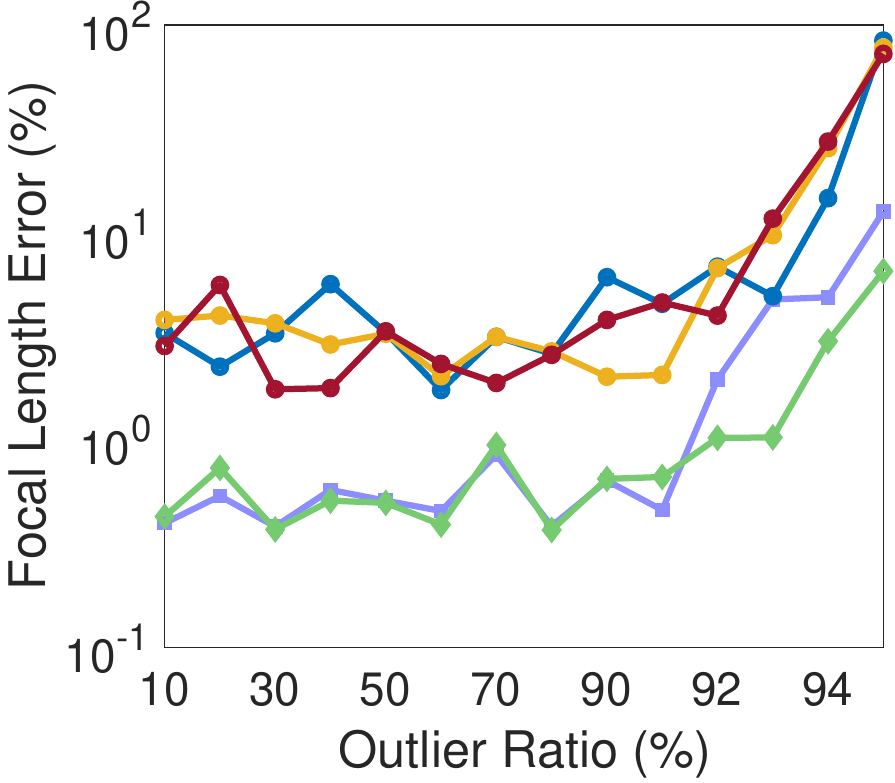}}
       % \hfill
    \subfloat[\scalebox{.90}{Rotation Success Rate}\label{fig:4DsuccessR}\\
    \tiny{Rotation Error $\leq 1$deg}]{%
       \includegraphics[height=.17\linewidth]{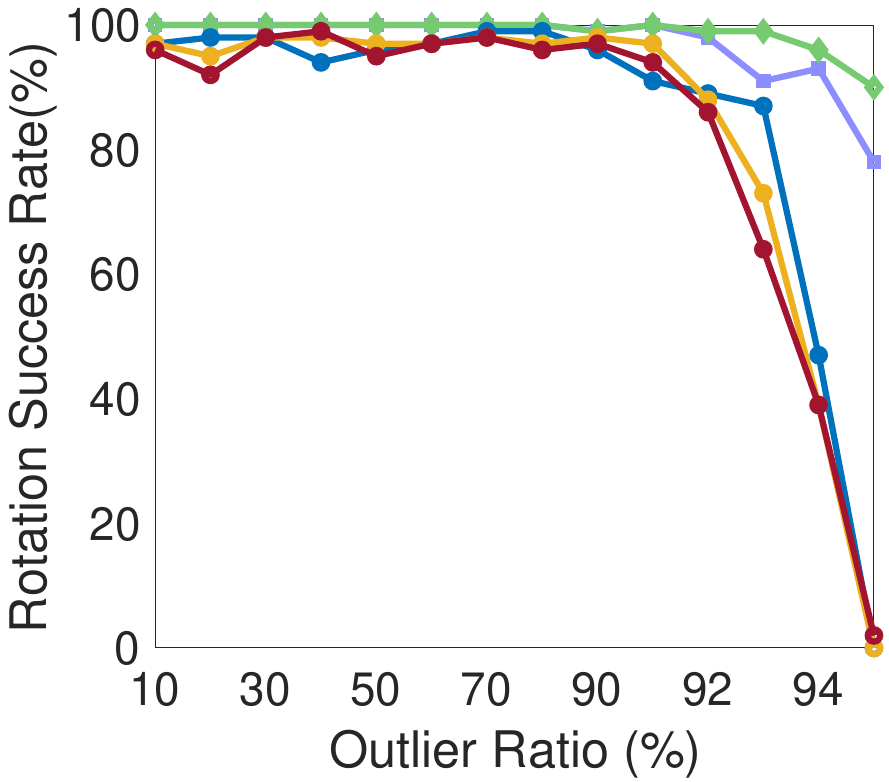}}
       % \hfill
    \subfloat[\scalebox{.90}{Focal Success Rate}\label{fig:4Dsuccessf}\\
    \tiny{Focal Length Error $\leq 5\%$}]{%
       \includegraphics[height=.17\linewidth]{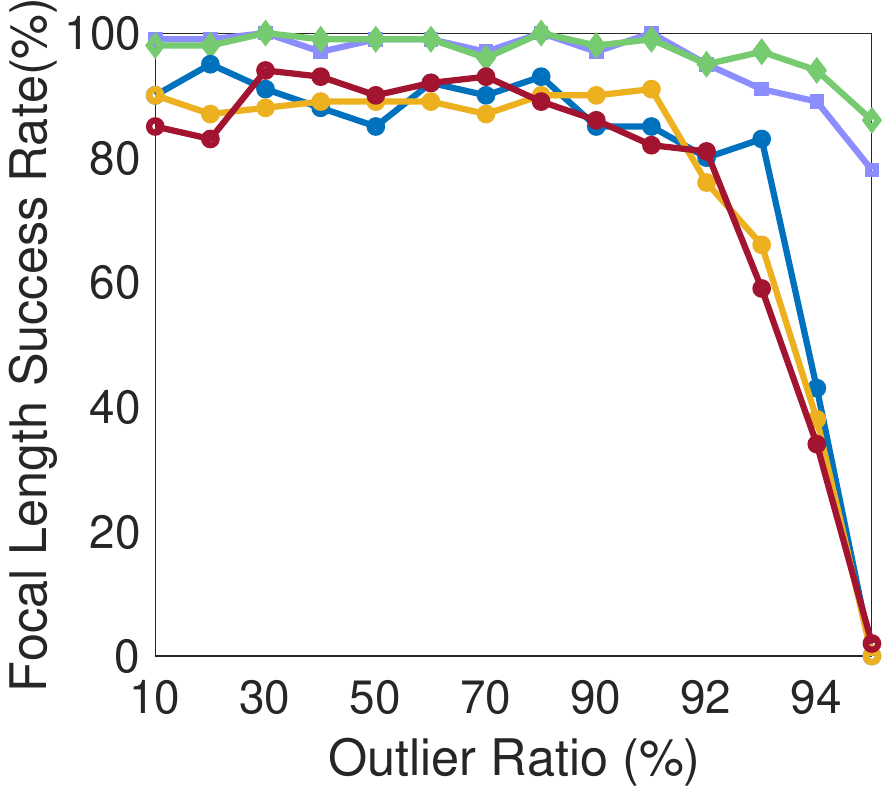}}
    \caption{Rotation and focal length estimation experiments (Section \ref{sec:4D}). ACM-3 archives $4_\times$ speed-up over plain BnB (\figurename~\ref{fig:4Dtime}). In the cases of large outlier ratios, MLESAC had a hard time meeting the confidence criterion. ACM-3 performs similar rotation (\figurename~\ref{fig:4DerrorR}) and focal length errors (\figurename~\ref{fig:4Derrorf}) with plain BnB. Furthermore, as a BnB-based method, ACM-3 also works well in large outlier ratios where MLESAC tend to fail (\figurename~\ref{fig:4DsuccessR}-\figurename~\ref{fig:4Dsuccessf}). ($100$ randomly generated noisy correspondences, $100$ trials)}
    \label{fig:4Dsimulation}
\end{figure*}

\subsection{Synthetic Experiments}
\noindent \textbf{Setup.} We compare plain BnB and ACM-3 on the synthetic data, and we furthermore include MLESAC into comparison as it is believed to be more robust to noise than vanilla RANSAC (the performance of RANSAC in our experiments is comparable to MLESAC and not shown here). We terminate MLESAC if it reaches a certain degree of confidence, say $90\%,95\%, 99\%$.

To generate synthetic data, we first randomly sample $100$ $3$D points having a distance between $4$ and $8$ to the world frame origin and then project them onto the camera frame. We uniformly sample the focal length in $[200,1500]$ and correspondingly set the canvas size to be $1480 \times 2160$ pixels. Rotation angles are uniformly sampled from $[-\pi/2, \pi/2]$. Pixel noise, randomly sampled from Gaussian distribution $\mathcal{N}(0, 0.5)$, is added to both coordinates of the image points.  

We set the inlier threshold to be $1$ pixel. The rotation error is measured as $\arccos((tr(\mathbf{R}_{gt}^T\hat{\mathbf{R}})-1)/2)$, where $\mathbf{R}_{gt}$ and $\hat{\mathbf{R}}$ are the ground truth and the estimated rotation matrices, respectively. The focal length error is defined as the absolute difference between the estimated and the ground truth focal length, divided by the ground truth focal length.

\noindent \textbf{Results.}  \figurename~\ref{fig:4Dsimulation} presents the comparison results over running time and estimation errors. It can be observed that ACM-3 is consistently faster than plain BnB with roughly a $4\times$ speed-up and provides comparably low errors as plain BnB in rotation and focal length. To be specific, when the outlier ratio is $90\%$ plain BnB needs about $41.33$ sec to converge while ACM-3 only needs $10.78$ sec on average.

As expected, ACM-3 as a BnB-based method still works well even for outlier ratios above over $90\%$ while MLESAC starts to fail (\figurename~\ref{fig:4DerrorR} and \figurename~\ref{fig:4Derrorf}). To further compare ACM-3 and MLESAC, we report their \textit{success rates} (a trial is a success if the rotation error is less or equal to $1$deg and focal length error is less or equal to $5\%$). Note that even in the extreme case where the outlier ratio is $94\%$, ACM-3, as well as plain BnB, provide less than $1$ deg error in over $80\%$ of the trials. Meanwhile, MLESAC can only succeed in around $40\%$ of the trials (\figurename~\ref{fig:4DsuccessR}). Similar observations can be made in \figurename~\ref{fig:4Dsuccessf}.
It is therefore intuitively clear that for such high outlier ratio cases when MLESAC meets the confidence the robust solution is barely valid. Note that the confidence level is set to 99\% as this leads to an execution time comparable to our algorithm for such extreme outlier scenarios. In summary, our methods ultimately outperform MLESAC in challenging scenarios, both in terms of accuracy and in terms of computational efficiency.

\section{Conclusion}\label{Sec:Conclusion}
%Todo: mention, often leading to 10x or more gain in computational efficiency

In this work, we have introduced a general strategy to speed-up the solution of globally optimal consensus maximization via branch-and-bound. The proposed strategy consists of---for an $n$-dimensional problem---simply branching over an $(n-1)$-dimensional space, and then solving for the remaining variable globally optimally using interval stabbing. Though the latter step takes on $\mathcal{O}(\log M)$ complexity, its cost is typically compensated for by the fact that much fewer volumes are created by branching over a smaller-dimensional space. Furthermore, embedding the globally optimal interval stabbing mechanism typically leads to tighter lower bounds, thereby causing earlier pruning of sub-optimal branches. We have validated our approach on four fundamental geometric registration problems, including camera resectioning, relative camera pose estimation,  point set registration, and rotation and focal length estimation. In the case of point set registration, we have further extended the application to the correspondence-less case by running the optimization over exhaustive match lists. This variant of the correspondence-less registration scenario is successfully solved owing to the algorithm's strong ability to handle extreme outlier ratios. The attained execution times often hind at potential use in real-time applications, thereby lifting the relevance of globally optimal consensus maximization from a pure validation tool to a viable solution for online processing.

% or



\ifCLASSOPTIONcaptionsoff
  \newpage
\fi



\bibliographystyle{IEEEtran}
\bibliography{references}


% \end{thebibliography}

% % biography section
% % 
% % If you have an EPS/PDF photo (graphicx package needed) extra braces are
% % needed around the contents of the optional argument to biography to prevent
% % the LaTeX parser from getting confused when it sees the complicated
% % \includegraphics command within an optional argument. (You could create
% % your own custom macro containing the \includegraphics command to make things
% % simpler here.)
% %\begin{IEEEbiography}[{\includegraphics[width=1in,height=1.25in,clip,keepaspectratio]{mshell}}]{Michael Shell}
% % or if you just want to reserve a space for a photo:

% \end{IEEEbiography}


% % insert where needed to balance the two columns on the last page with
% %\newpage



\vfill

\end{document}